\documentclass[ journal,twoside,final]{IEEEtran}
\usepackage[T1]{fontenc}

\ifCLASSINFOpdf
\else
\usepackage[dvips]{graphicx}
\fi
\usepackage{url}

\usepackage{amsmath}
\usepackage{amsthm}
\usepackage[cmintegrals]{newtxmath}

\usepackage{cite}
\usepackage{amsmath,amsfonts}
\usepackage{algorithm,float}
\usepackage{algorithmic}
\usepackage{graphicx}
\usepackage{textcomp}
\usepackage{xcolor}
\usepackage{subcaption}
\usepackage{stfloats}
\usepackage{float}
\usepackage{bm,bbm}
\newtheorem{theorem}{Theorem}
\newtheorem{lemma}{Lemma}
\newtheorem{assumption}{Assumption}
\newtheorem{corollary}{Corollary}
\newtheorem{remark}{Remark}
\usepackage{url}
\usepackage{color,soul}
\usepackage{mathrsfs}

\makeatletter
\newcommand*{\rom}[1]{\expandafter\@slowromancap\romannumeral #1@}
\makeatother

\newcommand{\mb}[1]{\mathbf{#1}}
\newcommand{\mc}[1]{\mathcal{#1}}

\newcommand{\bmat}{\begin{bmatrix}}
\newcommand{\emat}{\end{bmatrix}}

\newcommand{\VEC}[1]{\mb{#1}}

\newcommand{\state}{\VEC{s}}
\newcommand{\action}{\VEC{a}}   
\newcommand{\vectheta}{\bm{\theta}}
\newcommand{\bigtheta}{\bm{\Theta}}
\newcommand{\veclambda}{\bm{\lambda}}

\usepackage[acronym]{glossaries}

\newacronym{aoa}{AoA}{angles of arrival}
\newacronym{awgn}{AWGN}{additive white Gaussian noise}

\newacronym{bbu}{BBU}{baseband unit}

\newacronym{cdf}{CDF}{cumulative distribution function}
\newacronym{cee}{CEE}{channel estimation error}
\newacronym{cp}{CP}{cyclic prefix}
\newacronym{cpri}{CPRI}{common public radio interface}
\newacronym{cran}{C-RAN}{centralized radio access network}
\newacronym{csi}{CSI}{channel state information}

\newacronym{dft}{DFT}{discrete Fourier transform}
\newacronym{ds}{DS}{direction-selection}
\newacronym{du}{DU}{digital unit}

\newacronym{fft}{FFT}{fast Fourier transform}
\newacronym{fh}{FH}{Fronthaul}
\newacronym{phy}{PHY}{physical layer}

\newacronym{irc}{IRC}{interference rejection combining}
\newacronym{ici}{ICI}{inter-cell interference}

\newacronym{lls}{LLS}{lower-layer split}
\newacronym{los}{LoS}{line-of-sight}
\newacronym{ls}{LS}{least square}
\newacronym{lte}{LTE}{Long Term Evolution}

\newacronym{mimo}{MIMO}{multiple-input-multiple-output}
\newacronym{mmse}{MMSE}{minimum mean-square error}
\newacronym{mrc}{MRC}{maximum ratio combining}
\newacronym{mumimo}{MU-MIMO}{multiuser-MIMO}

\newacronym{nr}{NR}{New Radio}

\newacronym{pca}{PCA}{Principal Component Analysis}

\newacronym{ran}{RAN}{radio access network}
\newacronym{rf}{RF}{radio frequency}
\newacronym{rru}{RRU}{remote radio unit}

\newacronym{snr}{SNR}{signal-to-noise-power ratio}
\newacronym{sir}{SIR}{signal-to-interference-power ratio}
\newacronym{sinr}{SINR}{signal-to-interference-and-noise ratio}
\newacronym{svd}{SVD}{singular value decomposition}
\newacronym{tco}{TCO}{total-cost-of-ownership}

\newacronym{ue}{UE}{user equipment}
\newacronym{ula}{ULA}{uniform linear array}
\newacronym{ul}{UL}{uplink}

\newacronym{zf}{ZF}{zero-forcing}

\graphicspath{{../figures/}}
\newcommand{\norm}[1]{ \lVert#1 \rVert} 
\def\BibTeX{{\rm B\kern-.05em{\sc i\kern-.025em b}\kern-.08em
    T\kern-.1667em\lower.7ex\hbox{E}\kern-.125emX}}
\begin{document}

\title{Adaptive Stochastic ADMM for Decentralized Reinforcement  Learning in Edge Industrial IoT}

\author{
Wanlu~Lei, \IEEEmembership{Student Member, IEEE}, Yu~Ye, \IEEEmembership{Student Member, IEEE}, Ming~Xiao, \IEEEmembership{Senior Member, IEEE}, \\Mikael Skoglund, \IEEEmembership{Fellow, IEEE}, and Zhu Han, \IEEEmembership{Fellow, IEEE}
	\thanks{W. Lei is with Ericsson AB, Stockholm. W. Lei, Y. Ye, M. Xiao and M. Skoglund are with the Division of Information Science and Engineering, KTH Royal Institute of Technology, Stockholm, Sweden (e-mail: \{wllei, yu9, mingx, skoglund\}@kth.se).
	
	Z. Han is with the University of Houston, Houston, USA (email:
    zhan2@uh.edu). 
	}
}
\maketitle

\begin{abstract}
Edge computing provides a promising paradigm to support the implementation of Industrial Internet of Things (IIoT) by offloading tasks to nearby edge nodes. Meanwhile, the increasing network size makes it impractical for centralized data processing due to limited bandwidth, and consequently a decentralized learning scheme is preferable. 
Reinforcement learning (RL) has been widely investigated and shown to be a promising solution for decision-making and optimal control processes. For RL in a decentralized setup, edge nodes (agents) connected through a communication network aim to work collaboratively to find a policy to optimize the global reward as the sum of local rewards.  However, communication costs, scalability and adaptation in complex environments with heterogeneous agents may significantly limit the performance of decentralized RL. Alternating direction method of multipliers (ADMM) has a structure that allows for decentralized implementation, and has shown faster convergence than gradient descent based methods. Therefore, we propose an adaptive stochastic incremental ADMM (asI-ADMM) algorithm and apply the asI-ADMM to decentralized RL with edge-computing-empowered IIoT networks. We provide convergence properties for proposed algorithms by designing a Lyapunov function and prove that the asI-ADMM has $O(\frac{1}{k}) +O(\frac{1}{M})$ convergence rate where $k$ and $ M$ are the number of iterations and batch samples, respectively. Then, we test our algorithm with two supervised learning problems. For performance evaluation, we simulate two applications in decentralized RL settings with homogeneous and heterogeneous agents. The experiment results show that our proposed algorithms outperform the state of the art in terms of communication costs and scalability, and can well adapt to complex IoT environments.

\end{abstract}
\glsresetall


\section{Introduction}\label{sec:intro}
\subsection{Background}
The Industrial Internet of Things (IIoT), as a subset of IoT, interconnects a multitude of industrial devices, actuators, and people at work. 
The fourth industrial revolution aims to tackle a set of new technological challenges in industrial control, automation and intelligence \cite{liu2020latency}. 
Currently, many of the data from the IIoT networks are processed and stored off-site in remote areas (i.e., cloud). However, due to the limitation of network bandwidth, response latency and privacy, it is often inefficient or even impractical to send all the data to faraway central computing nodes. To address the aforementioned problems, edge computing empowered IIoT is proposed as a promising solution, in which edge nodes, such as sensors, actuators and small cells, are equipped with computation, storing and resource managing capability to process and store the data locally \cite{leong2020deep,she2020deep,lin2021deploying,nguyen2020trusted,liu2020over}. Edge computing for IIoT is also known as a decentralized cloud, or distributed cloud solution to address the drawbacks of cloud-centric models \cite{mach2017mobile,pan2017future}. 
Multiple edge nodes (IIoT devices) can work jointly to solve large-scale distributed tasks, and collaborate to make decisions to enable feasible and diverse IIoT applications. 
\begin{figure} [t]
\vskip -0.0 in
\centering
\includegraphics[width=86mm]{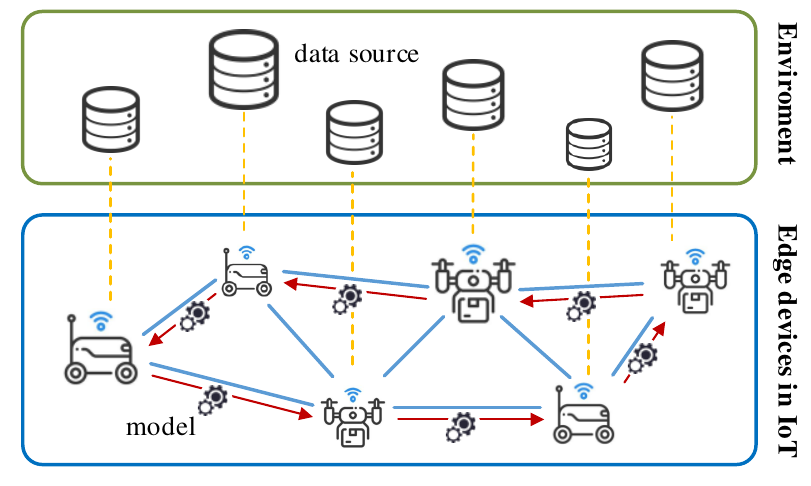}
\caption{Decentralized edge computing in IIoT networks.}
\label{fig:system model}
\vskip -0.25 in
\end{figure}

Machine learning (ML) is often employed in IIoT edge computing systems to analyze a large amount of data or obtain useful information for a variety of tasks \cite{ baek2019managing,lei2020deep,qian2020reinforcement}. Among various ML schemes, reinforcement learning (RL) has been intensively studied for decision-making and optimal control related applications in edge computing, e.g., IoT localization services \cite{mohammadi2017semisupervised}, wireless cloud controller \cite{de2020wireless}, and resource allocation for spectrum and computation with radio access technologies \cite{liu2020multi}. In RL, agents (nodes) take actions in a stochastic environment over a sequence of time steps, and learn an optimal policy to minimize the long-term cumulative cost from interacting with the environment. Though RL was first developed for a single-agent task, to facilitate the development in distributed computing, many practical RL tasks involve multiple agents operating in a distributed way \cite{shalev2016safe,liu2020multi,mnih2016asynchronous}. However, these tasks normally require frequent information exchanges between agents. With more devices deployed at the edge, the communication overhead can be very large, which becomes the bottleneck of overall performance. In addition to communication load, learning networks may have heterogeneous agents, where some agents have less computation power and thus slow down the overall convergence \cite{liu2020multi}.


Generally, ML with multiple nodes in decentralized computing can be formulated in a form where all distributed agents seek to collaboratively solve one optimization problem:
\begin{equation}
    f(\vectheta) = \min_{\vectheta} \sum_{i=1}^N f_i(\vectheta;\mc D_i),\label{eq:p1}
\end{equation}
where $f_i:\mathbb R^{p\times d} \to \mathbb R$ is the local loss function only known to agent $i$, and $\mc D_i$ is the corresponding local dataset. Variable $\vectheta$ is shared across all agents. Decentralized RL tasks with parametrized approximation also fit in the same form as (\ref{eq:p1}) only with the local loss as the expectation of individual cumulative costs.

\subsection{Related Work}
RL normally makes sequential decisions in dynamic situation where the environment evolves in time with uncertainties. The key idea of RL is learning through interacting with the environment. Generally modeled as a Markov decision process (MDP)\cite{sutton1998introduction}, the sequential-decision-making problems have been tackled by various RL algorithms, including Q-learning \cite{watkins1992q}, policy-gradient (PG) \cite{baxter2001infinite}, actor-critic methods \cite{konda2000actor}, etc. Due to huge success in Alpha Go \cite{silver2017mastering} defeating human players in the game of Go, many control-related applications in IIoT edge computing are developed using RL. Control tasks in IIoT networks are in a wide range, such as computation offloading \cite{liu2020multi}, resource allocation in fog networks \cite{hanzhu1}, and localization services in smart cities \cite{mohammadi2017semisupervised}. To exploit these applications in a decentralized setting without a central controller, frequent information exchange can occur between agents (edge nodes in IIoT edge computing). Besides, RL typically requires substantial historical data and computation resources for improving performance \cite{bertsekas2019reinforcement}. 

Decentralized RL in IIoT applications normally have two settings: parallel RL and multi-agent cooperative RL. Parallel RL is motivated by solving large-scale RL tasks that run in parallel on multiple learners. Parallel may have good scalability as well as the robustness of the multi-learner system. \cite{mnih2016asynchronous} introduces asynchronous methods and shows that parallel learners have a stabilizing effect on training processes since the training time reduces to half on a single multi-core CPU. \cite{nair2015massively} presents the massively distributed architecture for deep RL and shows that the performance surpasses most objects by reducing wall-time by an order of magnitude. In fully decentralized cooperative multi-agent RL (MARL), agents share a global state and each agent only observes its local loss. The goal of cooperative RL is to jointly minimize (maximize) global cost (reward). The work in \cite{zhang2018fully} is the first theoretical study of fully decentralized MARL. \cite{liu2020multi} proposes independent learner based multi-agent Q-learning for resource allocation in IoT networks. 

Most decentralized RL schemes mainly use gradient methods, which are directly extended from single-agent learning. 
Reference \cite{zhao2020distributed} applies the inexact ADMM approach in distributed MARL. However, the communication cost increases with the network size in \cite{zhao2020distributed}. \cite{hanzhu2} proposes game-based ADMM and shows that the convergence rate is independent of the network size. Another work in \cite{chen2019communication} proposes LAPG algorithm to reduce the communication overhead by adaptively skipping the gradient communication during iteration. However, the setting still involves a central controller and is based on the gradient descent method. 

Decentralized solvers in the ML optimization tasks, as in (\ref{eq:p1}), can normally be classified into primal and primal-dual methods. The primal method is commonly referred as gradient-based  \cite{nedic2009distributed, DGD,shi2015extra,rabbat2005quantized}. Each node averages its iterations from neighbors and descends along its local negative gradient. Normally, decentralized gradient-descent (DGD) \cite{DGD} and EXTRA \cite{shi2015extra} have good convergence rates with respect to its iteration number (corresponding to computation time). Moreover, gradient-based algorithms are shown to have constrained error bounds for constant step sizes \cite{DGD}, and can achieve exact convergence with diminishing step sizes at the price of slow convergence speed \cite{qu2019accelerated}. The primal-dual methods solve an equivalent constrained form of (\ref{eq:p1}) (see (\ref{eq:admmP2}) in Section.~\ref{sec:admm}). Many pioneering works, such as the distributed ADMM (D-ADMM) in \cite{DADMM}, the communication-censored ADMM (COCA) in \cite{COCA} and random-walk ADMM \cite{WADMM} are proposed to limit information sharing in each iteration. D-ADMM is similar as DGD, which requires each agent to collect information from all its neighbors. COCA can adaptively determine whether a message is informative during the optimization process. Following COCA, W-ADMM is an extreme instance where only one agent is randomly picked to be active per iteration. 

Apart from communication bottleneck, algorithm efficiency is also an important performance indicator. As for stochastic objectives, taking subsamples of data instead of full dataset can make optimization more efficient. In particular, stochastic gradient-descent (SGD)\cite{strang2019linear} and its variants have drawn a lot of research interests. Researchers have also applied variance reduction techniques for ADMM, such as \cite{liu2017accelerated,huang2019faster}. However, they all require large storage to keep past gradients, which can be problematic in large multitask learning. 
Moreover, the data samples in decentralized IIoT edge computing is generated at edge nodes and stored locally. Therefore, the data distribution can be non-i.i.d. The problem is even more pronounced in RL, in which the data sample distribution can change throughout the learning dynamics \cite{papini2018stochastic}. To address this issue, the state-of-the-art method, Adam \cite{adam}, uses the first-order gradient for update, and is computationally efficient, and requires little memory requirements. Thus Adam may be suited for large-scale learning. However, Adam is still an SGD-based method and cannot be well suited for complicated learning problems.

\subsection{Contributions}
Aiming at a communication efficient, scalable and adaptive decentralized RL scheme for edge-computing-empowered IIoT networks, we investigate the stochastic ADMM based RL solvers. In the following, we will firstly study the adaptive stochastic ADMM for solving (\ref{eq:p1}) and then apply the results to decentralized RL. 
The main contributions of this paper can be summarized as follows:
\begin{itemize}
    \item We propose a new adaptive stochastic incremental-ADMM (asI-ADMM) method for solving decentralized consensus optimization (\ref{eq:p1}), the updating order of which follows a predetermined order. We use the first-order approximation as well as proximal updates for primal variables to stabilize the convergence property. To further solve large deviation for stochastic objectives, we apply a weighted exponential moving average estimation of the true gradient and send this estimate as a token at each iteration. 
    \color{black}{\item We provide convergence properties for the asI-ADMM by designing a Lyapunov function. We prove that the asI-ADMM has $O(\frac{1}{k}) +O(\frac{1}{M})$ convergence rate, where $k$ denotes the iteration number and $M$ denotes the mini-batch sample size. We use numerical experiments for two regression problems, and the results show that the proposed algorithm requires fewer communication rounds and still achieves better optimum solution compared with state-of-the-art methods.
    \item We study two settings in decentralized RL: parallel and cooperative, and we formulate them into a consensus optimization problem. We adopt the asI-ADMM to decentralized RL and prove that asI-ADMM in RL can also achieve $O(\frac{1}{k}) + O(\frac{1}{M})$ by REINFORCE estimator. } 
    \item We conduct two empirical experiments in edge IIoT network setting: UAV localization and computing resource management with homogeneous and heterogeneous configurations, in which the latter has different initial states distribution and scaled reward functions across the agents. We show that the proposed asI-ADMM based algorithm in decentralized RL outperforms the benchmarks in terms of communication costs and is also well adaptive in complex setups.
\end{itemize}  

The rest of the paper is organized as follows. Section. \ref{sec:admm} presents the problem formulation as well as the description of the asI-ADMM algorithm. 
Convergence analysis is given in Section \ref{sec:admm}. Section \ref{sec:DRL} describes decentralized RL, and presents asI-ADMM based method for solving decentralized RL and its convergence analysis. Numerical results and conclusion are given in Sections \ref{sec:results} and \ref{sec:conclusion}, respectively.


Throughout the paper, we adopt the following notation: $\mathbb E [\cdot]$ denotes the expectation with respect to a set of variable $\bm{\zeta}_i^k = \{\zeta_{i,m}^k\}_M$. $\vert \cdot \vert$ is the absolute value. $\Vert \cdot \Vert^2$ denotes the Euclidean norm. $\nabla f(\cdot)$ denotes the gradient of a function $f(\cdot)$. $\langle\cdot\rangle$ denotes the inner product in a finite dimensional Euclidean space. $\vectheta^*\in\mc X$ denotes the optimal solution to (\ref{eq:p1}), where $\mc X \subset\mathbb R^m$ is the domain. Besides, we define $D_{\mc X} \triangleq \sup_{\vectheta_a,\vectheta_b\in\mc X}\norm{\vectheta_a-\vectheta_b}^2$.

\section{Problem Formulation and ADMM Based Algorithm}\label{sec:admm}
In this section, we reformulate the problem in (\ref{eq:p1}) for decentralized edge computing in IIoT networks and propose asI-ADMM to solve the reformulated decentralized consensus problem. 
\subsection{The Learning Problem with ADMM}

Consider a system with $N$ agents operating in a common environment. We are interested in decentralized settings without a central controller. Denote the decentralized network $\mc G:= (\mc N, \mc E)$, where $\mc N:= \{1,...,N\}$ represents the set of agents and $\mc E$ represents the set of communication links. 
Define $\bigtheta = [\vectheta_1, ..., \vectheta_N] \in \mathbb R^{mN}$, where $\vectheta_i$ is the parameter at the $i$-th agent, and the loss function $f_i$ captures the error of the model on local learning data. Thus, we can reformulate problem \eqref{eq:p1} as 
\begin{equation}\label{eq:admmP2}
    \min_{\bigtheta,z} \sum_{i\in\mc N} f_i(\vectheta_i;\mc D_i),~\text{s.t. } \mathbbm 1\otimes z-\bigtheta = \bm{0}, 
\end{equation}
where $z\in\mathcal{X}$, and $\mathbbm{1} = [1,...,1] \in \mathbb R^N$, and $\otimes$ is the Kronecker product. Here $\mc D_i$ is the local dataset at the $i$-th agent, which is collected from sensors such as drones or actuators. The augmented Lagrangian for problem (\ref{eq:admmP2}) is
\begin{equation}\label{eq:lagaragian}
    \mathcal{L}_{\rho}(\bm{\Theta},\bm{\lambda},z) = \sum_{i\in\mc N}f_i(\vectheta_i; \mc D_i) + \langle \bm{\lambda}, \mathbbm 1\otimes z - \bigtheta \rangle + \frac{\rho}{2} \Vert \mathbbm 1\otimes z - \bigtheta\Vert^2,
\end{equation}
where $\veclambda = [\lambda_1,...,\lambda_N] \in\mathbb R^{mN}$ is the dual variable, while $\rho>0$ is a constant parameter. Following I-ADMM \cite{ye2020privacy}, with guaranteeing $\sum_{i\in \mc N} (\bm{\theta}_i^0-\frac{\lambda_i^0}{\rho})=\bm{0}$, the updates of $\bm{\Theta}$, $\bm{\lambda}$ and $z$ at the $(k+1)$-th iteration are given by

  \begin{subequations}
	\begin{align}		
	&\bm{\theta}_i^{k+1}:= \left\{\begin{aligned}
	&\arg \min_{\bm \theta_i}  \mathcal{L}_{\rho} (\bm{\theta}_i,\bm{\lambda}^k,z^k ),~ i=i_k ;~~~~~ \\
	&\bm{\theta}_i^{k },~i\neq i_k;
	\end{aligned}   \right. \label{eq4a}\\
	&\lambda_i^{k+1}:= \left\{\begin{aligned}
	& \lambda_i^{k} +  \rho \gamma  ( z ^{k}-\bm{\theta}_{i}^{k +1}  ),~i=i_k ;\\
	&\lambda_i^{k },~i\neq i_k ;\label{eq4b}
	\end{aligned}   \right. \\
	&  z^{k+1} := z^{k} + \frac{1}{N} \left[ \left(\bm{\theta}_{i_k}^{k+1}- \frac{\lambda_{i_k}^{k+1}}{\rho}  \right) - \left(\bm{\theta}_{i_k}^{k }- \frac{\lambda_{i_k}^{k }}{\rho}\right)  \right].\label{eq4c}  
	\end{align}	
\end{subequations}

 In the stochastic update of (\ref{eq4a}), a mini-batch samples $\bm{\zeta}_i^k$ are drawn from dataset $\mathcal{D}_i$. To be more specific, $\bm{\zeta}_i^k\triangleq\{\zeta_{i,m}^k\}_{m=1}^M$ is the set of the randomly selected mini-batch samples and $M$ is the size of the mini-batch. To reduce the computation load, the stochastic first-order approximation $f_{i}(\bm{\theta}_{i};\mc D_{i}) \approx \langle  G_{i}(\vectheta_{i}^k;\bm{\zeta}_{i}^k),\bm{\theta}_{i}-\bm{\theta}_{i }^k\rangle$ will be adopted, where $ G_i(\vectheta_i^k;\bm{\zeta}_i^k)$ is the mini-batch stochastic gradient given by 
\begin{equation}
 G_i(\vectheta_i^k;\bm{\zeta}_i^k) = \frac{1}{M}\sum_{m=1}^M {\nabla} f_i(\bm{\theta};\zeta_{i,m}^k).   
\end{equation}
Though gradient $G_i(\vectheta_i^k;\bm{\zeta}_i^k)$ can be calculated with a sampling method, other sources of noise from stochastic objectives can cause large deviation in the result \cite{yuan2020stochastic}. Hence similar to Adam \cite{adam}, we apply the first-moment estimation of gradients as 
\begin{equation}
		    \mu^{k+1}:= \eta^k \mu^{k} + (1-\eta^k) G_{i_k }(\vectheta_{i_k }^k;\bm{\zeta}_{i_k}^k).\label{eq:mu_update}
\end{equation}
Here $\mu^{k+1}$ is the weighted exponential moving average (EMA) estimation of the true gradient, i.e., the first moment (the mean) of $\nabla f_i(\vectheta_i; \mc D_i)$ with a hyper-parameter $\eta^k\in[0,1)$ controlling the exponential decay rate of the moving averages. The constant $\eta^k$ can be set adaptively to determine the decay rate of past gradients. When $\eta^k\to0$, the algorithm will effectively eliminate the moving memory
\cite{strang2019linear}. Since $\mu^{k+1}$ is an estimation for the gradient $\nabla f_{i_k}(\bm\theta_{i_k}^k)$, from (\ref{eq:mu_update}), the upper bound of estimation variance can be measured by 
\begin{equation}
    \begin{aligned}
       \mathbbm{E}[ \|\mu^{k+1} -\nabla f_{i_k}(\bm\theta_{i_k}^k) \|^2] \leq&2\mathbbm{E}[\|  \nabla f_{i_k}(\bm\theta_{i_k}^k) - G_{i_k }(\vectheta_{i_k }^k;\bm\zeta_{i_k}^k)\|^2]     \\& 
        +2(\eta^k)^2\mathbbm{E}[ \|\mu^k-G_{i_k}(\bm\theta_{i_k}^k;\bm\zeta_{i_k}^k)  \|^2] .   
        \end{aligned}\label{eq:muv}
\end{equation}
The first term of (\ref{eq:muv}) is the variance of stochastic gradient, the upper bound of which is normally relevant to the mini-batch size. While the second term is the variance of stochastic gradient from the former EMA. To bound the variance of EMA, we need to construct an efficient bound for the second term. Hence, we propose to apply the following adaptive rule for the selection of $\eta^k$ as
\begin{align}
    \eta^k=\left\{
    \begin{aligned}
        &\bar{\eta},~(\bar{\eta})^2\|\mu^k-G_{i_k}(\bm\theta_{i_k}^k;\bm\zeta_{i_k}^k)  \|^2\leq \frac{\iota^2}{M};   \\
        &\sqrt{\frac{\iota^2}{M}}\frac{1}{\|\mu^k-G_{i_k}(\bm\theta_{i_k}^k;\bm\zeta_{i_k}^k)  \|},~\text{otherwise},
    \end{aligned}
    \right.\label{eq:eta}
\end{align}
where $\bar{\eta}\in[0,1)$ and $\iota\in\mathbbm{R}$ are pre-determined hyper parameters. From (\ref{eq:eta}), it is guaranteed that the second term of (\ref{eq:muv}) is bounded by a constant $\frac{\iota^2}{ M}$. Besides, the adaptive rule (\ref{eq:eta}) indicates that $\eta^k \leqslant \bar{\eta}$, which controls the impact of past gradients $\mu^k$ according to the mini-batch size $M$. The recent work in \cite{chen2020cada} applies a different control rule for adaptive stochastic gradients that can be implemented to save communication loads.

To further stabilize the convergence behavior for the inexact augmented Lagrangian method, we introduce proximal update for primal variables. The update of $\bm\Theta^{k+1}$ can be rewritten as
\begin{subequations}
	\begin{align}	\label{eq:opti_cond}	
	&\bm{\theta}_i^{k+1}:= \left\{\begin{aligned}
	&\arg \min_{\bm \theta_i} 
	\hat{\mathcal{L}}_{\rho} (\bm{\theta}_i,\bm{\lambda}^k,z^k ) ,~ i=i_k ;~~~~~ \\
	&\bm{\theta}_i^{k },~i\neq i_k,
	\end{aligned}   \right. 
	\end{align}	
\end{subequations}
\begin{algorithm}[t]

	\caption{Adaptive stochastic I-ADMM (asI-ADMM) } 
	\begin{algorithmic}[1]
    \STATE \textbf{initialize}: $\{\bm{\theta}_i^0 =  \lambda_i^0=z^0=\mu^0=\bm{0}, k=0,\bar{\eta}, \iota |i\in\mathcal{N}\}$; 
		\FOR{$k=0,1,...$}
		\STATE agent $i_k = k \mod N + 1$ do: 		 
		\STATE \textbf{receive} token $\mu^k,z^{k }$;   
		\STATE uniformly and randomly pick $M$ samples $\bm \zeta_{i_k}^k\subseteq\mc D_i$, and compute stochastic gradient $G_{i_k}(\vectheta_{i_k}^k;\bm \zeta_{i_k}^k)$ ;
		\STATE \textbf{choose} $\eta^{k}$ according to (\ref{eq:eta});
		\STATE \textbf{update} $\mu^{k+1}$ according to (\ref{eq:mu_update});
		\STATE \textbf{update} $\bm{\theta}_{i_k}^{k+1}$ by (\ref{eq:primal_first_order});
		\STATE \textbf{update} $\lambda_{i_k}^{k+1}$ according to (\ref{eq4b});
		\STATE \textbf{update} $z^{k+1}$ according to (\ref{eq4c});
		\STATE \textbf{send} $\mu^{k+1}$ and token $z^{k+1} $ to agent $i_{k+1}= (k+1) \mod N + 1$;  
		\ENDFOR 		
	\end{algorithmic} \label{alg:aspiADMM}
\end{algorithm} 

\begin{figure} [t]
\centering
\includegraphics[width=80mm]{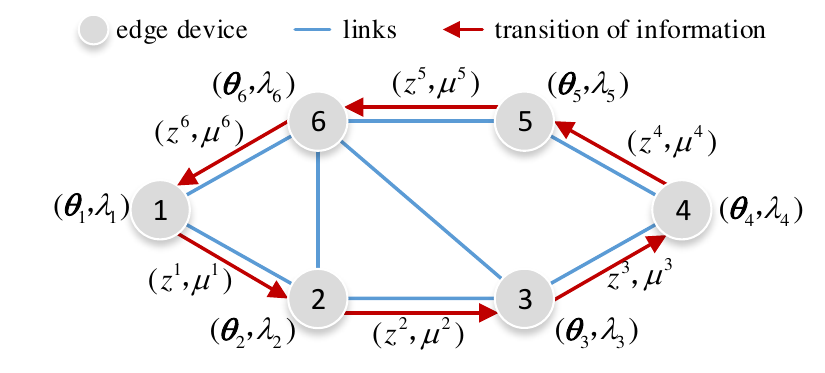}
\caption{asI-ADMM multi-agent communications in decentralized setup.}
\label{fig:system model}
\vskip -0.1 in
\end{figure}
\noindent
where the approximated Lagrangian function is give by
\begin{equation}
\begin{aligned} 
    &\hat{\mathcal{L}}_{\rho} (\bm{\theta}_i,\bm{\lambda}^k,z^k )= \langle\mu^{k+1},\bm{\theta}_i-\bm{\theta}_i^k \rangle\\
    &~~~~~~~~~~~~~~~~+ \frac {\rho }{2}  \|  z ^k-\bm{\theta}_i +\frac{\lambda_i^k}{\rho}  \big\|^{2}+\frac{\tau }{2} \big\| \bm{\theta}_i - \bm{\theta}_i^k \|^2.
    \end{aligned}\label{eq:primal_first_order}
\end{equation}

  The asI-ADMM is formally presented in Algorithm \ref{alg:aspiADMM}, where the agents are activated in a predetermined circulant pattern. Particularly, we assume the update order $\langle i_k \rangle_{k\geq0}$ repeats in a Hamiltonian circle: $1\to2\to3\cdots\to N\to1\cdots$, as shown in Fig. \ref{fig:system model}. The global variable $z^{k+1}$ and gradient estimation $\mu^{k+1}$ get updated at agent $i_k$ and passed as tokens to its neighbour $i_{k+1}$ through Hamiltonian cycle. When $\{\eta^k = 0|k=1,2,...\}$, the algorithm reduces to the vanilla stochastic incremental ADMM (sI-ADMM) as in \cite{chen2020coded}. Comparing with sI-ADMM, asI-ADMM constructs stochastic gradient $\mu^{k+1}$ based on the information $G_{i_k}(\vectheta_{i_k}^k;\bm{\zeta}_{i_k}^k)$ and $\mu^{k}$, while sI-ADMM only considers the current mini-batch gradient. Using more local information, asI-ADMM can lead to more accurate estimation of the full gradient. Moreover, it does not require additional storage and computation resources.

\begin{remark}
The algorithm asI-ADMM can also be extended to an online version when data samples $\bm\zeta_{i}^k$ are generated in real-time or drawn from a distribution $\mathcal{P}_i$ of data source in the environment at step 5 instead of a fixed dataset $\mc D_i$.  
\end{remark}
We will describe more details in Section \ref{sec:DRL} for applying asI-ADMM for decentralized RL that implements in an online manner. 

\subsection{Convergence Analysis}
In what follows, we will analyze the convergence of Algorithm \ref{alg:aspiADMM}. To facilitate the analysis, we first introduce definitions and assumptions as follows. 
\begin{assumption} The graph $\mathcal G$ is connected, and there exists a Hamiltonian cycle.
\end{assumption}
According to \cite{findingHamilton}, Hamiltonian cycle can be found in polynomial time. If the Hamiltonian cycle does not exist, we can instead consider a shortest path in the connected graph $\mathcal{G}$ as in \cite{ye2020privacy}, the analysis of which however is not included here due to space limitation.

\begin{assumption}\label{asup:L-smooth} The local loss function $f_i(\vectheta)$ in (\ref{eq:admmP2}) is lower bounded over $\vectheta\in\mathcal{X}$, and $f_i(\vectheta)$ is coercive over $\vectheta$. Each local function $f_{i}(\vectheta)$ is \textit{L-Lipschitz differentiable}, i.e., for any $\vectheta_1,\vectheta_2 \in\mathcal{X}$,
    \begin{equation}
    \begin{aligned}
     \|\nabla f_i(\bm \theta_1)-\nabla f_i(\bm \theta_2)   \| \leq L \| \bm \theta_1 - \bm \theta_2  \| ,~\forall i\in\mathcal N. 
    \end{aligned}
    \end{equation} 
    \end{assumption}
    With Assumption \ref{asup:L-smooth}, we obtain the following useful inequality for loss function $f_i(\bm \vectheta)$,
\begin{equation}
    f_i(\vectheta_1) \leqslant f_i(\vectheta_2) + \langle \nabla f_i(\vectheta_2), \vectheta_1 - \vectheta_2 \rangle + \frac{L}{2}\norm{\vectheta_1 - \vectheta_2}^2, \label{eq:Lsmooth2}
\end{equation}
for any $\vectheta_1, \vectheta_2 \in \mathcal{X}$.

Then we give some standard assumptions regarding stochastic optimization problem in (\ref{eq:admmP2}).
\begin{assumption}(\textit{Unbiased estimation})
    For smooth function $f_i(\vectheta)$, there exists a stochastic first-order oracle that returns a noisy estimation to the gradient of $f_i(\bm{\theta})$, and the unbiased estimation satisfies
    \begin{equation}
        \mathbb E[\nabla f_i(\vectheta;\zeta_{i,m})] = \nabla f_i(\vectheta).
    \end{equation}
    Then we also have for mini-batch stochastic estimation that 
    \begin{equation}
        \mathbb E[ G_i(\vectheta;\bm{\zeta}_i)] = \nabla f_i(\vectheta).
    \end{equation}
\end{assumption}

 \begin{assumption}(\textit{Bounded gradient}) \label{assump2} Stochastic gradient of loss function $f_i(\vectheta)$ is bounded. That is, there exists a constant $\delta>0$ such that 
     \begin{equation}
         \max_{1\leq i \leq N} \sup_{\theta\in{\Theta}} \norm{\nabla f_i(\vectheta;\zeta_{i,m})}^2 \leqslant \delta^2.
     \end{equation}
     The variance of stochastic gradient $\nabla f_i(\vectheta;\zeta_{i,m}) $ is bounded by $\sigma^2$ for all $\vectheta$, 
     and then we also have 
     \begin{equation}
         \mathbb E [\norm{  G_i(\vectheta;\bm{\zeta}_{i}) - \nabla f_i(\vectheta)}^2] \leqslant \frac{\sigma^2}{ M}, ~  \forall \vectheta \in\mathcal{X}.
     \end{equation}
     
\end{assumption}
Given Assumption 4, we can derive the following bound for the EMA estimation. 
\begin{lemma}\label{lemma:muBound} The upper bound of estimation variance for weighted EMA $\mu^k$ is given by
 \begin{equation}
\mathbbm{E} [\norm{\mu^{k+1}-\nabla f_{i_{k}}(\vectheta_{i_{k}}^{k})}^2 ] \leqslant \frac{2}{M}(\iota^2+\sigma^2).
 \end{equation}
 \end{lemma}
 \begin{IEEEproof}
 Applying the adaptive control rule in (\ref{eq:eta}), we have $\mathbbm{E}[(\eta^k)^2 \|\mu^k-G_{i_k}(\bm\theta_{i_k}^k;\bm\zeta_{i_k}^k)  \|^2] \leq \frac{\iota^2}{M}$, based on which the final result is obtained from (\ref{eq:muv}).
 \end{IEEEproof}
\begin{remark}
With applying the adaptive rule in (\ref{eq:eta}), the estimation variance $\mathbbm{E}[\norm{\mu^{k+1}-\nabla f_{i_{k}}(\vectheta_{i_{k}}^{k})}^2] $ is determined by the mini-batch size $M$. As $M$ increases, $\eta^k$ gets smaller where $G_i(\vectheta_i^k;\zeta_i^k)$ dominants the weighted EMA estimation $\mu^{k+1}$. This is reasonable since the stochastic gradient $G_i(\vectheta_i;\zeta_i)$ is more close to the true gradient $\nabla f_i(\vectheta_i)$ with more samples. As $M\to \infty$, it gives $\eta^k=0$ from adaptive rule (\ref{eq:eta}), and the weighted EMA is derived as $\mu^{k+1}=G_i(\vectheta_i^k;\bm{\zeta}_i^k)=\nabla f_i(\vectheta_i)$. 
\end{remark}
A common way for proving convergence of a method or an algorithm is by finding Lyapunov functions \cite{lyapunoc}.  A Lyapunov function can be interpreted as "energy" that decreases geometrically with each iteration of the algorithm, with an energy of zero corresponding to reaching the optimal solution of the problem. To prove the convergence of Algorithm 1, we first define a Lyapunov function as
\begin{equation}
   \mathcal{V}^k \triangleq \mathbbm{E}\big[\mathcal L_{\rho}(\bm \Theta^k,\veclambda^k, z^k)\big],~k\in \mathbbm{N}^+.
\end{equation}

The property for this Lyapunov function is summarized as follows.
\begin{lemma} \label{lem:uppB}
For any $k\in\mathbbm{N}^+$, Lyapunov function $\mc V^k$ is lower bounded as below 
\begin{equation}
 	    \mathcal{V}^k \geqslant F^* -\frac{2N(1-\bar{\eta})\delta^2}{\rho(1-{(\bar{\eta}})^k)} -  \left[\frac{2\tau^2}{\rho}-\frac{3\rho}{8}-\rho(\gamma-1)^2 \right]ND_{\mathcal X},
 	\end{equation}
where $F^*= \sum_{i=1}^N f_i(\vectheta^*) $.
\end{lemma}
\begin{IEEEproof}
See Appendix A the Proof of Lemma \ref{lem:uppB}. 
\end{IEEEproof}
Based on Lemmas 1 and 2, we finally present the convergence results for asI-ADMM.
\begin{lemma} \label{lem:converg}
Suppose sequence $\{\bm{\Theta}^k,\veclambda^k, z^k|k=1,...,K\} $ is generated from Algorithm \ref{alg:aspiADMM}, and
let $\tau \geqslant \frac{L +\rho-1}{2}$, $\rho \geqslant 1$ and $\gamma > 4N$, then we have
\begin{equation}
      \frac{1}{K}\sum_{k=0}^K (\norm{\vectheta_{i_k}^k-\vectheta_{i_k}^{k-1}}^2 + \norm{z^k-z^{k-1}}^2 )  \leqslant \frac{\mathcal{V}^0 - \mathcal{V}^* }{K\kappa}+ \frac{\iota^2+\sigma^2}{M\kappa},
 	\end{equation}
where $\kappa=\text{min}(\chi,\varphi)$ with $\chi   = \frac{\rho-2L+2\tau+1}{2} - \frac{2\rho}{\gamma}$ and $\varphi = \frac{N\rho}{2}- \frac{2N\rho^2}{\gamma}$.
\end{lemma}

\begin{IEEEproof}
See Appendix B the Proof of Lemma \ref{lem:converg}. 
\end{IEEEproof}
Lemma \ref{lem:converg} indicates that sequence $\{\norm{\bigtheta^k-\bigtheta^{k-1}}^2 + \norm{z^k-z^{k-1}}^2|k=1,2,...\}$ converges with increasing iteration $k$. Besides, due to the incremental update (\ref{eq4c}), we can conclude that the all the sequences $\{\norm{\bigtheta^k-\bigtheta^{k-1}}^2,\norm{z^k-z^{k-1}}^2, \norm{\veclambda^k-\veclambda^{k-1}}^2|k=1,2,... \}$ are convergent. Next, based on Lemmas \ref{lem:uppB} and \ref{lem:converg}, we give the convergence properties of asI-ADMM regarding the derivatives.
\begin{theorem} \label{theorem1}
 Suppose sequence $\{\bm{\Theta}^k,\veclambda^k, z^k|k=1,...,K\} $ is generated from Algorithm \ref{alg:aspiADMM}, and following the conditions given in Lemma \ref{lem:converg}, we have
\begin{equation}\label{eq:theo1}
    \frac{1}{K}\sum_{k=0}^{K}\mathbbm E \big[\norm{\nabla \mc L(\bigtheta^k ,\veclambda^k ,z^k )}^2 \big] \leqslant \frac{\varepsilon(\mc V^0 - \mc V^*)}{K\kappa}+ \frac{\varepsilon(\iota^2+\sigma^2)}{M\kappa},
\end{equation}
where $\varepsilon =5 + 5L^2 + 5\tau^2 +15\rho^2+ 10\rho^2N^2$.
\end{theorem}
\begin{proof}
See Appendix C the Proof of Theorem 1.
\end{proof}
\begin{remark}
Theorem 1 shows that with $K,M\to \infty$, the sequence $\{\bigtheta^k,\bm{\lambda}^k,z^k\}$ in asI-ADMM will converge to a \textbf{stationary point} of augmented Lagrangian (\ref{eq:lagaragian}). Note that Theorem 1 provides a sufficient condition to guarantee the convergence of the proposed asI-ADMM. Besides, the result presented in Theorem 1 indicates a $O(\frac{1}{k})+O(\frac{1}{M})$ convergence rate for asI-ADMM, which demonstrates that a larger batch size promotes faster convergence. 
\end{remark}
\begin{remark}
The proof of Theorem 1 also indicates that the convergence rate of asI-ADMM can be described as  $O(\frac{1}{k})+O(\frac{1}{M})+O(\frac{N^2}{k}) + O(\frac{N^2}{M})$. With the network size scaling up, the convergence speed is partially degraded by the square of network size $N$. 
\end{remark}

{\color{black}{So far we have presented a general framework for decentralized optimization, which can be applied to different settings with objectives $\{f_i\}$ accordingly. In the following, we will adopt the proposed method for decentralized RL with edge computing in IIoT networks. }  }

\section{Decentralized RL with ADMM} \label{sec:DRL}
We introduce essential background of RL and the policy search based methods in a decentralized setting in Subsections \ref{subsec_drl} and \ref{subsec:pg}, respectively. By formulating the decentralized RL problem into a consensus form similar to (\ref{eq:admmP2}), we adopt the proposed asI-ADMM method to decentralized RL and provide corresponding convergence analysis for the proposed algorithm in Subsection \ref{subsec:rl_admm}. 

\subsection{Decentralized RL}\label{subsec_drl}


Consider $N$ agents located on a time-invariant communication network denoted by $\mc G \! \!:= (\mc N, \mc E)$. We then define the MDP with networked agents in a decentralized setting. A networked multi-agent MDP is characterized by tuple $(\mc S, \mc A, \mc P, \beta, \alpha, g_{i\in \mc N})$, where $\mc S$ is the environmental state space and $\mc A$ is the action space; and $\mc P:\mc S \times \mc A \times \mc S \to [0,1]$ is the state transition probability of the MDP; and $\alpha \in (0,1)$ and $\beta$ are the discounting factor and the initial state distribution, respectively; and $g_i \! \! :\mc S \times \mc A \to \mathbb R$ is the local loss function of agent $i$. For different RL tasks in a IIoT edge computing, local loss function $g_i$ can be designed in different forms in order to guide the agents converging toward the objective. For example, $g_i$ can be defined based on the reciprocal duration of the job to minimize the average job slowdown in task offloading optimization. In an IIoT localization problem, the local loss for each agent can be the distance error to the target. In addition to the six-tuple, stochastic policy $\pi:\mc S \to \mc A$ maps the state to a distribution of all possible actions given the current state $\state_t$. We use $\pi(\action \vert \state)$ to denote the stochastic policy as the conditional probability of all possible actions given the state $\state$. In decentralized RL, the objective is to collaborative find the optimal policy $\pi$ to minimize the cumulative discounted loss over all agents.

We consider two decentralized multi-agent MDP settings, \textit{parallel agents} and \textit{collaborative agents}. In the parallel setting, agent $i$ aims to solve an independent MDP $(\mc S_i, \mc A_i, \mc P_i, \beta_i, \alpha, g_{i\in \mc N})$, in which the local action and state spaces as well as the transition probability of every agents are the same, i.e., $\mc S_i = \mc S, \mc A_i = \mc A$ and $\mc P_i = \mc P, \forall i \in \mc N$. Under such setting, multiple agents in the network aim to solve a common stochastic policy $\pi$. However, the loss function and the initial state distribution may be different across agents. Nevertheless, they are still quantities drawn form the same distribution satisfying $\beta(\state) =\mathbb E[\beta_i(\state)]$ and $g(\state,\action) = \mathbb E[g_i(\state,\action)]$ for all $(\state,\action)\in \mc S \times \mc A$. In a collaborative multi-agent setting, all agents share a global state $\state_t\in\mc S$ and $\action_t = (a_{1,t},a_{2,t},...,a_{N,t})$ is a joint action and its space can be denoted as $\mc A = \prod_{i\in\mc N} \mc A_i$, and $\action_t  \in \mc A$. Different from parallel setting, the joint action determines the transition probability to the next state $\state_{t+1}$ as well as the local loss function $g_i(\state_t,(a_{1,t},a_{2,t},...,a_{N,t}))$. This means $\pi$ is a joint policy with $\pi = (\pi_1,...,\pi_N)$ concatenating all local policies $\{\pi_i\}_{i\in\mc N}$. 

Consider the $\alpha$-discounted problem and discrete time $t\in \mathbb N$, policy $\pi$ can generate an infinite horizon state-action trajectory $\zeta_i:=\{\state_0, \action_0,\state_1, \action_1,...\}$ with $\state_t \in \mc S$ and $\action_t \in \mc A$. In decentralized RL, the objective is to collaboratively find the consensus optimal policy $\pi$ that minimizes the sum of discounted cumulative loss over all agents in the network, i.e.,
\begin{equation}
    \min_{\pi} \sum_{i\in\mc N} J_i(\pi),~\text{s.t. }J_i(\pi) = \mathbb E _{\zeta_i\sim\mathbb P(\cdot\vert \pi)}  \left\{  \sum_{t=0}^{\infty} \alpha^t g_i(\state_t, \action_t) \right\}, \label{eq:loss_pi}
\end{equation}
where $J_i(\pi)$ is the cumulative loss for agent $i$. The expectation in \eqref{eq:loss_pi} takes over the random trajectory $\bm{\zeta}$ of agent $i$ under policy $\pi$. We denote the probability of generating trajectory $\zeta_i$ given policy $\pi$ as 
\begin{equation}
    \mathbb P(\zeta_i \vert \pi) = \mathbb P(\state_0) \prod_{t=0}^{\infty} \pi(\action_t\vert\state_t) \mathbb P(\state_{t+1}\vert\state_t,\action_t),\label{eq:policy_prob}
\end{equation}
where $\mathbb P(\state_0)$ is the probability of initial state being $\state_0$, which is chosen randomly with regard to different initial state distribution of agent $i$, and $\mathbb P(\state_{t+1}\vert\state,\action)$ is the transition probability from state $\state_t$ to state $\state_{t+1}$ by taking action $\action_t$. 
\vspace*{-\baselineskip}

\subsection{Policy Gradient with ADMM}\label{subsec:pg}
Policy search is a subfield in RL, which focuses on finding optimized parameters for a given policy parametrization \cite{deisenroth2013survey}. Learning a policy is often easier than learning a model and its environment. Thus, model-free policy search methods are used more often than model-based methods in many applications. Among several update strategies, 
such as policy gradient (PG) \cite{williams1992simple}, inference approaches  and information-theoretic approaches, 
PG-based methods have been recognized as one of the most pervasive RL algorithms especially for tasks with large-scale state-action spaces. Classical PG such as REINFORCE \cite{williams1992simple} and G(PO)MDP \cite{baxter2001infinite} both use Monte-Carlo sampling to estimate the gradient. In our algorithm design and performance analysis, we will leverage the REINFORCE and give details in the following.

The stochastic optimization problem in \eqref{eq:loss_pi} has very nice differentialability properties that are lacking in the original deterministic form as in \cite{bertsekas2019reinforcement}. PG methods restrict the search for the best performing policy to a subset of parametrized policies. Particularly, policy $\pi$ is parametrized by parameters $\bm{\theta}\in\mathbb R^m$, which is denoted as $\pi(\cdot\vert\state;\bm{\theta})$, or $\pi(\bm{\theta})$ for simplicity. Note that we assume $\pi(\bm{\theta})$ is a discrete distribution in what follows. 
Accordingly, the problem in \eqref{eq:loss_pi} can be rewritten as
\begin{equation}
        \min_{\bm{\theta}} \sum_{i\in\mc N} J_i(\bm{\theta}),~\text{s.t. }J_i(\bm{\theta}) = \mathbb E _{\zeta_i\sim\mathbb P(\cdot \vert \bm{\theta})}  \left\{  \sum_{t=0}^{\infty} \alpha^t g_i(\state_t, \action_t) \right\}, \label{eq:loss_theta}
\end{equation}
where $J_i(\bm{\theta})$ is the long-term discounted loss of the parametric policy $\pi(\bm{\theta})$, and $\mathbb P(\cdot\vert\bm{\theta})$ is the probability distribution of sample trajectories $\bm{\zeta}_i$ under the policy $\pi(\bm{\theta})$. 
Then we may use a general gradient method \cite{bertsekas2019reinforcement} for solving the problem in \eqref{eq:loss_theta}. 
The method requires that $\pi(\bm{\theta})$ is differentiable with respect to $\bm{\theta}$. It relies on a convenient gradient formula, known as the \textit{log-likelihood trick} and involves the natural logarithm of the sampling distribution. Thus combining (\ref{eq:policy_prob}) and (\ref{eq:loss_theta}), the gradient of cumulative loss $J_i(\bm{\theta})$ of each agent can be calculated as
\begin{equation}
\begin{aligned}
        \nabla J_i(\bm{\theta})  &= \nabla    \left[\sum_{\vectheta\in\mc X} \mathbb P(\zeta\vert\vectheta)\sum_{t=0}^{\infty} \alpha^t g_i(\state_t, \action_t)   \right] \\ 
       &=  \sum_{\vectheta\in\mc X}\mathbb P(\zeta\vert\vectheta)   \left[\frac{\nabla \mathbb P(\zeta\vert\vectheta)}{\mathbb P(\zeta\vert\vectheta)} \sum_{t=0}^{\infty} \alpha^t g_i(\state_t, \action_t)  \right]\\
       & = \mathbb E_{\zeta_i \sim \mathbb P(\cdot\vert\bm{\theta})}  \left\{ \sum_{t=0}^{\infty} \nabla   [ \log \pi(\action_t\vert\state_t;\bm{\theta}) ] \alpha^t g_i(\state_t, \action_t)  \right \}. \label{eq:pg}
\end{aligned}
\end{equation}
%
To avoid costly full gradients computation (sometimes can be infeasible due to infinite spaces) of $\nabla J_i(\bm{\theta})$, REINFORCE adopts an Monte Carlo estimator. We denote $\zeta_{i,m} = (\state_0^{i,m}, \action_0^{i,m},\state_1^{i,m},\action_1^{i,m},...,\action_{T-1}^{i,m},\state_T^{i,m})$ as the $m$-th $T$-slot trajectories generated by policy $\pi(\cdot\vert\bm{\theta})$ at agent $i$. Hereby, the unbiased estimator of $\nabla J_i(\vectheta)$ with the $m$-th generated trajectory is given as follows
\begin{align}
    &\hat{\nabla}_m J_i(\bm{\theta};\zeta_m) = \notag\\
     &~~~\quad\sum_{t=0}^{T} \nabla  \log \pi (\action_t^{[i,m]}\vert\state_t^{[i,m]};\bm{\theta} )   \left[\sum_{t=0}^{T} \alpha^t g_i\big(\state_t^{[i,m]},\action_t^{[i,m]}\big)  \right], \label{eq:pg_sample}
\end{align}
and the corresponding mini-batch PG is given as
\begin{equation}
    d_i(\vectheta;\bm{\zeta}_i) = \frac{1}{M}\sum_{m=1}^M \hat{\nabla}_m J_i(\bm{\theta};\zeta_{i,m}).\label{eq:pg_minibatch}
\end{equation}

Finally we introduce policy parametrization for $\pi(\vectheta)$. Since we only deal with discrete action space, to ensure differentiability of $\pi$ with respect to $\vectheta$, we apply soft-max approximation for policy function:
\begin{equation}
    \pi(\action\vert\state;\vectheta) = \frac{e^{\vectheta^T\phi(\state,\action)}}{\sum_{u\in\mc A}e^{\vectheta^T\phi(\state,u)}},
\end{equation}
where $\phi(\state)$ is a bounded feature mapping from state space $\mc S$ to $\mathbb R^d$ and $\phi(\state,\action)$ denotes the value of action $\action$ in $\phi(\state)$. Moreover, the Gaussian policy is often applied in the control task where action spaces are continuous \cite{papini2018stochastic}.
 
\begin{algorithm}[t] 
	\caption{asI-ADMM for decentralized RL} 
	\begin{algorithmic}[1]
    \STATE \textbf{initialize}: $\{\bm{\theta}_i^0 =  \lambda_i^0=z^0=\mu^0=\bm{0}, k=0,\bar{\eta}, \iota |i\in\mathcal{N}\}$
		\FOR{$k=0,1,...$}
		\STATE agent $i_k = k \mod N + 1$ do: 		
		\STATE collect $M$ trajectories $\bm \zeta_{i_k}^k$ with parameter $\vectheta_{i_k}^k$, compute mini-batch PG component $d_{i_k}(\vectheta_{i_k}^k;\bm{\zeta}_{i_k}^k)$ according to (\ref{eq:pg_minibatch});
		\STATE \textbf{receive} tokens $\mu^{k}$ and $z^{k }$;  
		\STATE \textbf{choose} $\eta^{k}$ according to (\ref{eq:eta});
		\STATE \textbf{update} $\mu^{k+1}$ according to (\ref{eq:mu2}); 
		\STATE \textbf{update} $\bm{\theta}_{i_k}^{k+1}$ according to (\ref{eq4a});
		\STATE \textbf{update} $\lambda_{i_k}^{k+1}$ according to (\ref{eq4c});
		\STATE \textbf{update} $z^{k+1}$ according to (\ref{eq4b});
		\STATE \textbf{send} token $z^{k+1} $ and $\mu^{k+1}$ to agent $i_{k+1}= (k+1) \mod N + 1$;  
		\ENDFOR 
	\end{algorithmic} \label{alg:RLEI_ADMM}
\end{algorithm} 

\subsection{asI-ADMM for Decentralized RL and Convergence Analysis}\label{subsec:rl_admm}
Our goal is to find an optimal consensus policy to minimize overall loss across the agents. Thus, (\ref{eq:loss_theta}) has the same form as the optimization problem in (\ref{eq:admmP2}). However, for online optimization such as RL, training samples are generated in real-time by interacting with the environment, which may need faster response to varying environments. Thus, we propose to use online asI-ADMM for decentralized RL problem in (\ref{eq:loss_theta}). The essential part of the algorithm is to use the weighted EMA estimator for true PG:
\begin{equation}
    \mu^{k+1} = \eta^k\mu^k + (1-\eta^k)d_i(\vectheta;\bm{\zeta}_i),\label{eq:mu2}
\end{equation}
where $d_i(\vectheta;\bm{\zeta}_i)$ is obtained from (\ref{eq:pg_minibatch}), which is of critical importance in decentralized RL. Within heterogeneous agents appear in the environment, initial state distribution and local loss can be varying across the agents. Sending $\mu$ as a token can be helpful for consensus convergence. We can see later in the simulation that the adaptive method converges faster and obtain the better local optimum especially in the heterogeneous environment. 

We describe asI-ADMM based method for decentralized RL in Algorithm \ref{alg:RLEI_ADMM}. The updating order is predetermined by the Hamiltonian cycle. In the case where the communication network is time-varying with agents moving during learning process, then updating order can follow random walk as in \cite{WADMM}. Note that the update steps for $\vectheta,\lambda$ and token $z$ at steps 7, 8, and 9 are independent of sample generation at step 4. Thus these two parts can be ran in parallel at current agent and next agent, respectively. Next we will show the convergence properties of asI-ADMM based algorithm for decentralized RL. We firstly introduce following assumptions that serve as stepping stone for subsequent analysis:
\begin{assumption} \label{assump:rl1}
For each state-action pair $(\bm{s},\bm{a})$, the local loss
is bounded as $g_i(\bm{s},\bm{a})\in [0,\overline{g}_i]$, and for each parameter $\bm{\theta}\in \mathbb{R}^m$, the cumulative loss of learner $i$ is bounded as $J_i(\bm{\theta})\in [0,\frac{\overline{g}_i}{1-\gamma}]$.
\end{assumption}
\begin{assumption} \label{assump:rl2}
For each state-action pair $(\bm{s},\bm{a})$ and any policy parameter $\bm{\theta}$, there exists constants $c_1$ and $c_2$ such that
\begin{equation}
    \| \nabla \log \bm{\pi}(\bm{a}|\bm{s};\bm{\theta})  \|\leq c_1,~ \bigg| \frac{\partial^2}{\partial \bm{\theta}_l  \partial \bm{\theta} _k } \log \bm{\pi}(\bm{a}|\bm{s};\bm{\theta})  \bigg|\leq c_2,
\end{equation}
where $\theta_l$ and $\theta_k$ denote the $l$-th and the $k$-th entries of $\bm{\theta}$.
\end{assumption}
Assumption 5 requires local loss and cumulative loss to be bounded. The assumption is common in RL algorithms. Assumption 6 requires the gradient of $\log\pi(\action\vert\state;\vectheta)$ and its partial derivatives to be bounded. This can be satisfied by a wide range of stochastic policies including soft-max policies and Gaussian policies. We show that Assumptions 5 and 6 are sufficient to guarantee the smoothness of the objective function in (\ref{eq:loss_theta}) as follows. 
 \begin{lemma}\label{lemma:L_smmoth}
Under Assumptions \ref{assump:rl1} and \ref{assump:rl2}, the accumulated loss $J_i(\bm{\theta})$ for agent $i$ is $\bar{L}_i$-smooth, i.e.,
\begin{equation}\label{eq:L-smooth1}
    \|\nabla J_i(\bm \theta_1)-\nabla J_i(\bm \theta_2)   \| \leq L_i \| \bm \theta_1 - \bm \theta_2  \| ,~\forall \bm{\theta}_1,\bm{\theta}_2\in\mathbb{R}^m,
\end{equation}
where
\begin{equation}
   \bar{L}_i:= \left( c_1^2 + c_2 + \frac{2\kappa c_1^2}{1-\kappa}  \right) \frac{\kappa \overline{g}_i}{(1-\kappa)^2},\label{eq:L-smooth}
\end{equation}
and (\ref{eq:L-smooth1}) is equivalent to 
    \begin{equation}
    J_i(\bm \theta_1) \leqslant  J_i(\bm \theta_2) +  \langle  \nabla J_i(\vectheta_2), \vectheta_1 - \vectheta_2   \rangle + \frac{\bar{L}_i}{2} \norm{\vectheta_1-\vectheta_2}^2.
    \end{equation}
\end{lemma}
\begin{IEEEproof}
This can be proved as Lemma \ref{lemma:L_smmoth} in \cite{chen2019communication}.
\end{IEEEproof}
Given that REINFORCE is an unbiased PG estimator, the gradients in (\ref{eq:pg_sample}) and (\ref{eq:pg_minibatch}) are naturally satisfied with Assumptions 3 and 4 as in Section \ref{sec:admm}. 
Then based on Lemma \ref{lemma:L_smmoth}, the convergence property for Algorithm \ref{alg:RLEI_ADMM} can be concluded as follows.
\begin{corollary}
The sequence $\{\bm{\Theta}^k,\bm{\lambda}^k,z^k \vert k = 1,...,K\}$ generated from Algorithm \ref{alg:RLEI_ADMM} satisfies convergence property (\ref{eq:theo1}) but with different Lipschiz constants $\{\bar{L}_i|i\in\mathcal{N}\}$ as in (\ref{eq:L-smooth}).  
\end{corollary}
\begin{IEEEproof}
The proof is similar as that of Theorem 1.
\end{IEEEproof}
Note that Corollary 1 indicates a $O(\frac{1}{k}) + O(\frac{1}{M})$ convergence rate for Algorithm 2. 

\section{Numerical Experiments}\label{sec:results}
We will evaluate the performance of proposed stochastic algorithms on two decentralized consensus optimization problems: 1) Decentralized regression problems; 2) Decentralized RL, in Subsections \ref{subsec:result_regre} and \ref{subsec:restul_drl}, respectively. For the simulation, we generate the connected network $\mc G$ with $N$ agents and $\vert\mc E\vert = \frac{N(N-1)}{2}\omega$ links, where $\omega$ is the network connectivity ratio. This ensures a Hamiltonian cycle in $\mc G$.

\subsection{Decentralized Regression Problems}\label{subsec:result_regre}
   In what follows, we will consider the decentralized regression problems: least square and logistic regression. The former aims to solve (\ref{eq:admmP2}) with a loss function
   \begin{equation}
       f_i(\vectheta;\bm{\zeta}_i) = \frac{1}{M}\sum_{m=1}^{M}\norm{\vectheta_i^T o_{i,m} - t_{i,m}}^2,
   \end{equation}
where $\bm{\zeta}_i = \{o_{i,m},t_{i,m}\}_M$ is the data samples drawn from dataset $\mc D_i$. In simulations, the entries of input $o_{i,.m}\in \mathbb R^{10}$ and the output measurement $t_{i,m}\in\mathbb R$ follows $t_{i,m} = \vectheta^* o_{i,m} + e_{i,m}$, where $e_{i,m} \sim \mc N(0,\sigma^2 I)$ is random noise. The loss function for logistic regression is given as
\begin{equation}
    f_i(\vectheta;\bm{\zeta}_i) = \frac{1}{M_i}\sum_{m=1}^M \log(1+\exp(-t_{i,m}\vectheta_i^To_{i,m})),
\end{equation}
where $t_{i,m}\in\{-1,1\}$. Sample feature $o_{i,m}$ follows $\mc N(0,1)$ distribution. To generate $t_{i,m}$, we first generate $\vectheta\in\mathbb R^2$, and for each sample generated $v_{i,m}\sim \mc U(0,1)$, if $v_{i,m}\leqslant (1+\exp(-\vectheta^T o_{i,m})^{-1}$ then $t_{i,m} = 1$, otherwise $-1$. We set $r$ as the ratio for mini-batch samples.

We evaluate the convergence of proposed algorithm with state-of-the-art approaches regarding the accuracy defined by
\begin{equation}
    \text{accuracy} = \frac{1}{N}\sum_{i=1}^N \frac{\norm {\vectheta_i^k - \vectheta^*}^2}{\norm{\vectheta_i^0 - \vectheta^*}^2},
\end{equation} 

 where $\bm{\theta}^*$ is the optimal solution of (\ref{eq:admmP2}). For decentralized ridge and logistic regressions, the parameters for each method are set as follows: the stepsize $\gamma = \{0.5,0.05\}$ for EXTRA \cite{shi2015extra}; $\rho=\{0.01,1\}$ for D-ADMM \cite{DADMM};  $\alpha=\{0.5,0.05\}$ for DGD \cite{DGD}; $\rho=\{3,1\}$ for W-ADMM \cite{WADMM}; $\rho=\{3,1\}, \tau=\{0.5,0.2\}$ for proximal I-ADMM (prox. I-ADMM) and proximal stochastic (prox. sI-ADMM), and mini-batch ratio is set $r=0.1$ for prox. sI-ADMM; $\rho = \{3,1\}, \tau = \{0.2,0.2\}, \bar{\eta} = \{0.9,0.9\}$ and $\iota^2=\{10,10\}$ for asI-ADMM, respectively.
 

 The results of accuracy over communication costs for ridge regression and logistic regression are shown in Fig. \ref{fig:least_square}. Among all compared algorithms, the incremental based algorithms including our proposed algorithm, mini-batch method sI-ADMM and WADMM are more communication efficient than the gossip-based benchmarks, such as EXTRA, DADMM, DGD and COCA. We can see that there is gap between I-ADMM where a full batch is used at each iteration and stochastic I-ADMM methods where only a mini-batch is used for updating (with ratio $r=0.1$ in our case). Our proposed asI-ADMM converges faster than other methods and it also converges to a better optimal solution than sI-ADMM by using the same amount of samples. This is because the proposed asI-ADMM uses more information about past gradients in $\mu$ while sI-ADMM uses only current sample information. To keep the advantage of this gradient memory, a large $\bar{\eta}$ is needed (close to 1)\cite{strang2019linear,adam}. 


    \begin{figure}[t] 
     	\centering
     	\vskip -0.05  in
    	\subfloat[ ]{\hspace*{1mm}\includegraphics[width =45mm]{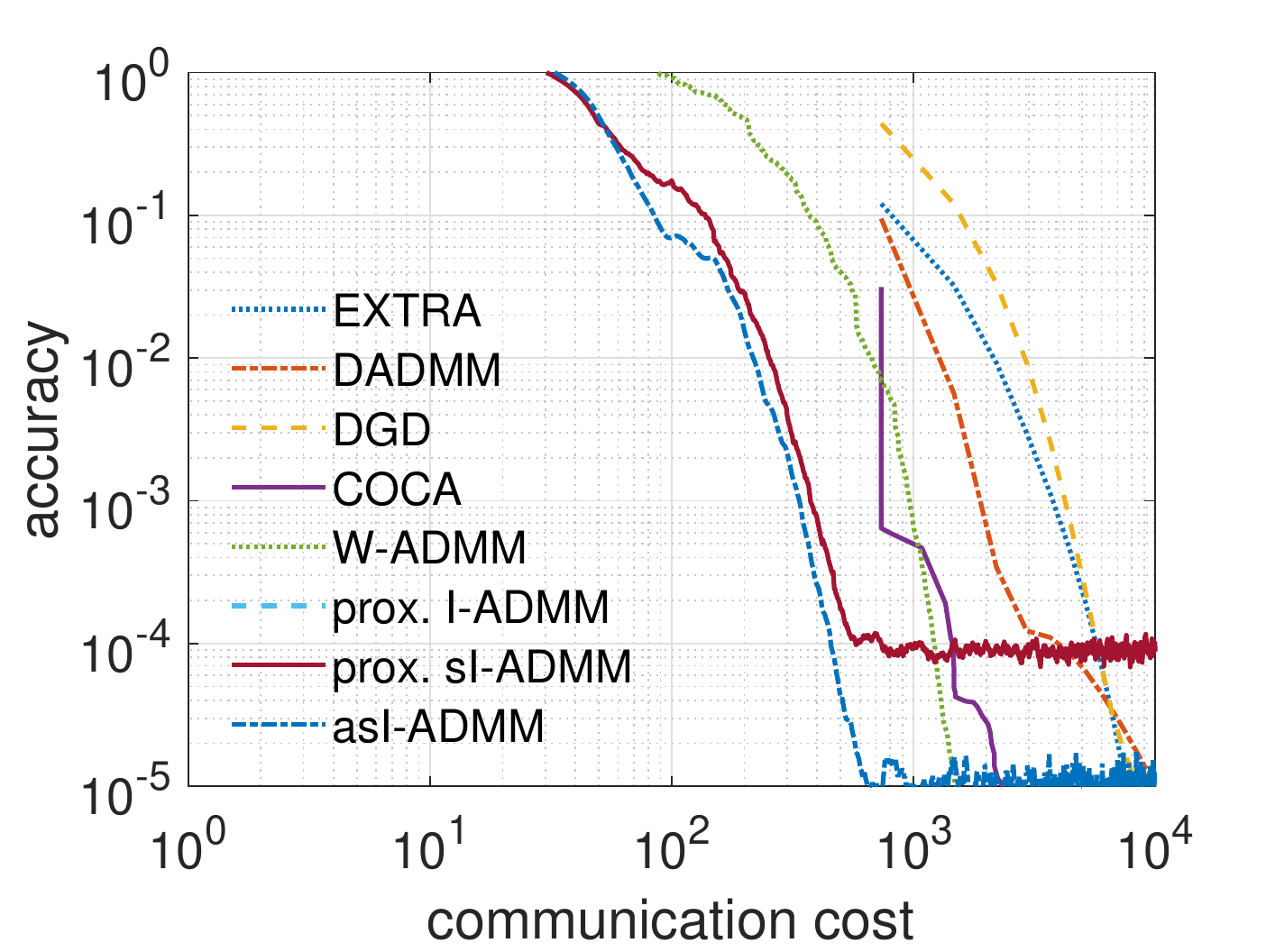}\label{fig_a}}
     	\subfloat[ ]{\hspace*{-2mm}\includegraphics[width =45mm]{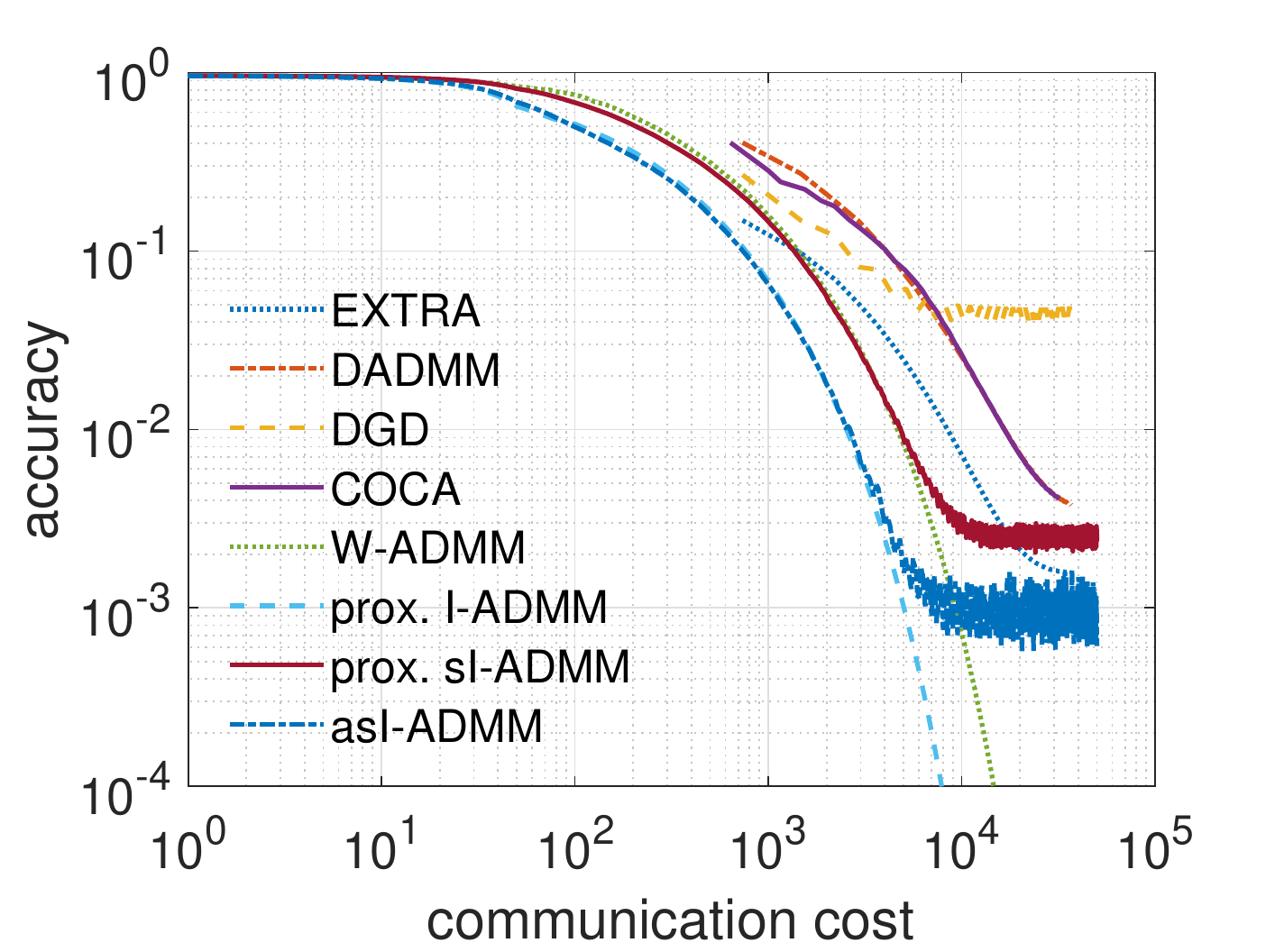}\label{fig_b}}
     	\vskip -0.1  in
     	\caption{ $T=50,\omega=0.3,r=0.1$: (a) Decentralized least square regression; (b) Decentralized logistic regression}
     	\label{fig:least_square}
     	\vskip -0.05 in
     \end{figure}

 \subsection{Decentralized RL}\label{subsec:restul_drl}
 To evaluate performance of asI-ADMM based algorithm in decentralized RL, we conduct two numerical experiments: localization task and computation resource management task.

 We modify the environment in \cite{UAVrl} to a decentralized RL setting where $N$ agents (UAVs) aim to work together to reach a specific target. Each agent can choose a set of four actions \textit{\{north, south, west, east\}} as shown in Fig. \ref{fig:target_loc}. The state is the positions of UAVs in the searching grid. Assuming power-law path-loss model, the power of signal received by the $i$-th agent is $P_i = l_0 \frac{P_0}{d_i^\nu} + e_i$, where $\frac{1}{l_0}$ is the path-loss at the reference distance of 1m, $d_i^\nu$ is the distance between the target and the $i$-th agent and $\nu$ is the power loss exponent. And $e_i$ is a random variable (r.v.), which accounts for the estimation error. We assume the scenario contains only light-of-sight components. The estimated position of agents can be obtained as in \cite{UAVrl}. The reward function is defined as:
 \begin{align}
     g_i(t) = \left\{
     \begin{aligned}
         & r_i, ~\text{if $d_i(t) < d_0$}; \\
    & -d_i(t), ~ \text{otherwise},
     \end{aligned}
     \right.
 \end{align}
where $d_0$ is a distance threshold, and $r_i$ is a positive scalar representing the priority level for agent $i$. 
 In homogeneous configuration, all agents have the same initial distribution and equal priority with their proximity to the target. 
 
 \subsubsection{Target localization}
In the experiments, we generate time-invariant communication network $\mc G$ to ensure the updating order predetermined by a Hamiltonian cycle. The discounting factor is set to $\alpha=0.99$ in all tests. $\omega$ is set to 0.3 for a 5-agent setup and 0.8 for 10-agent setup. For each episode, all algorithms terminate after $T=50$. The control parameter $\iota^2 = 10$ for all experiments. The targeted local policy of each agent $\pi(\vectheta_i)$ is parameterized by a linear function. The parameters for channel model are given as $l_0 = 20.7, \nu = 3.04$ and $e_i \sim \mc N(0,\sigma^2 I)$ \cite{UAVrl}.
  \setcounter{figure}{4}   
 \begin{figure*}[t] 
     	\centering
     	\vskip -0.0  in
     	\subfloat[ ]{\hspace*{-2mm}\includegraphics[width=45mm]{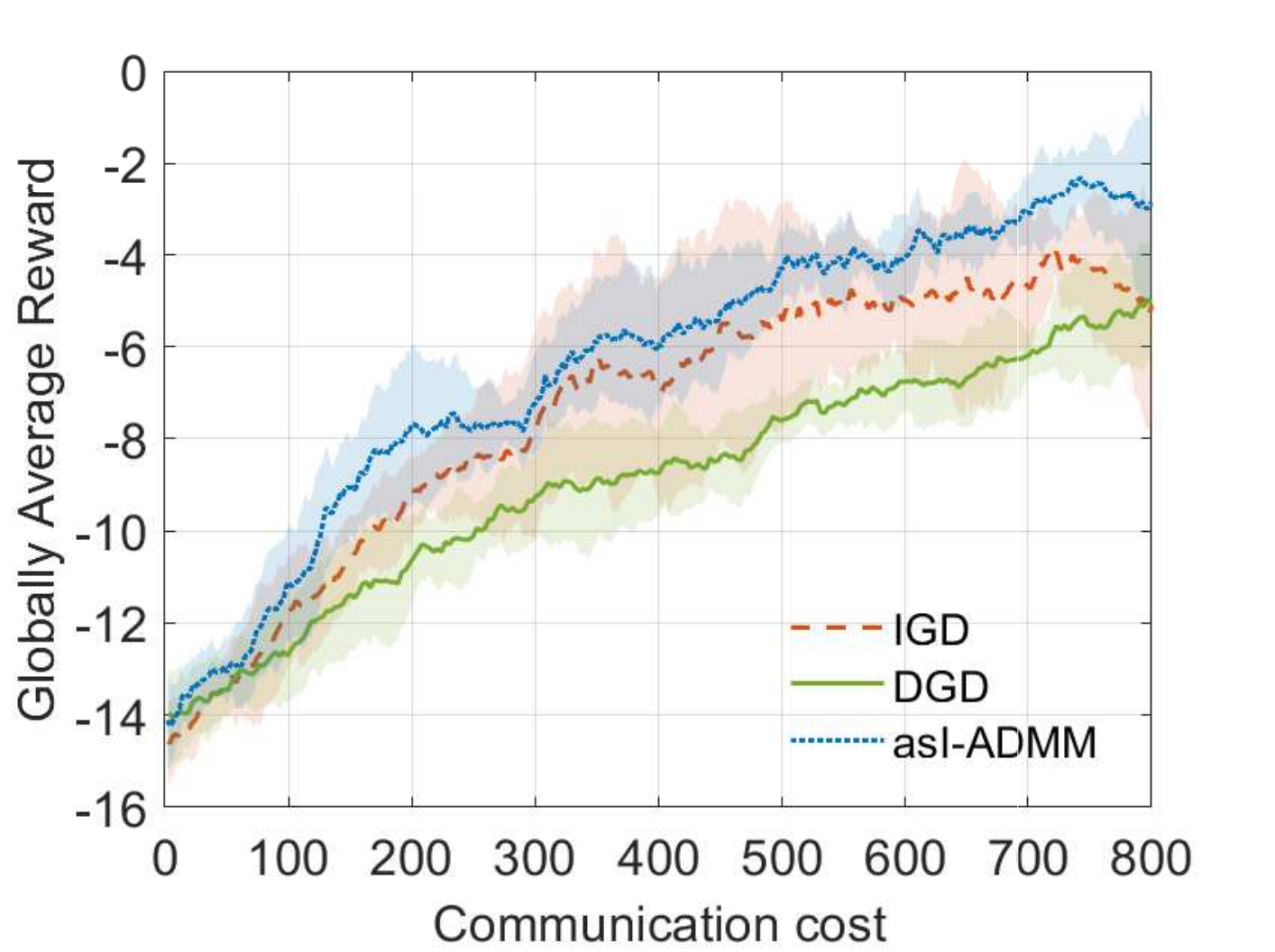}\label{fig1_b}}
     	\subfloat[ ]{\hspace*{-1mm}\includegraphics[width=45mm]{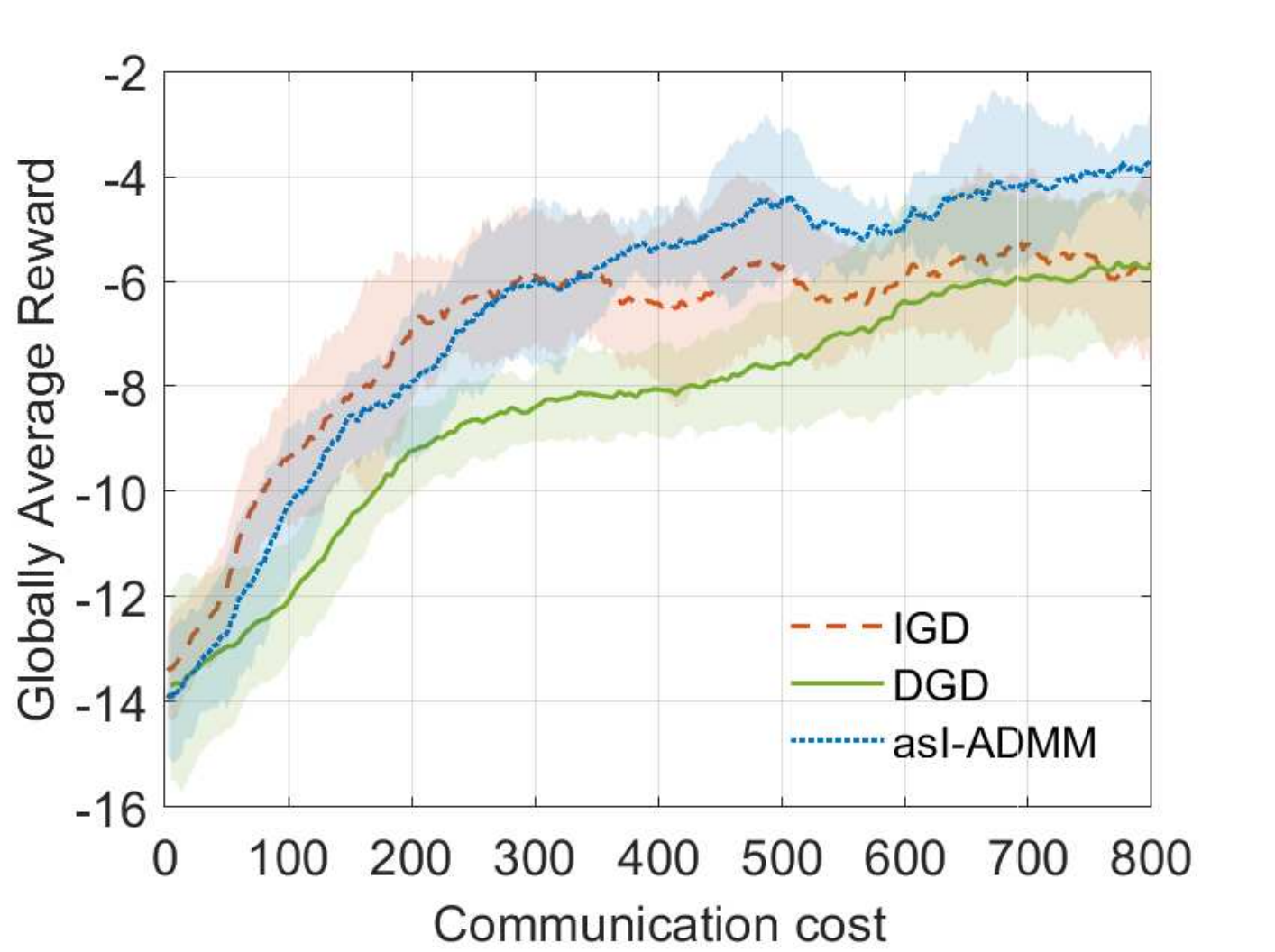}\label{fig1_d}}
     	\subfloat[ ]{\hspace*{-1mm}\includegraphics[width =45mm]{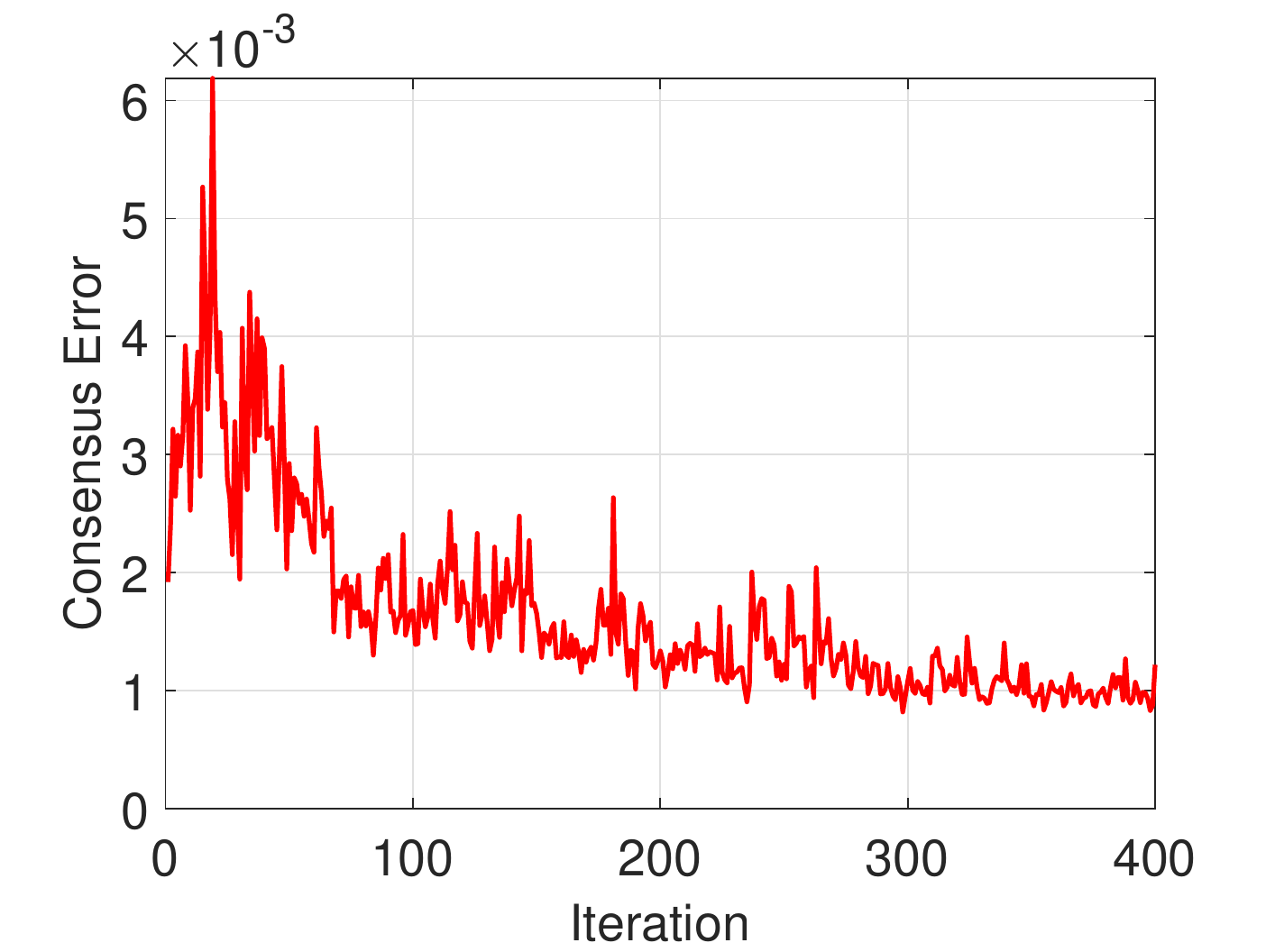}\label{fig_a}}
     	\subfloat[ ]{\hspace*{-1mm}\includegraphics[width=45mm]{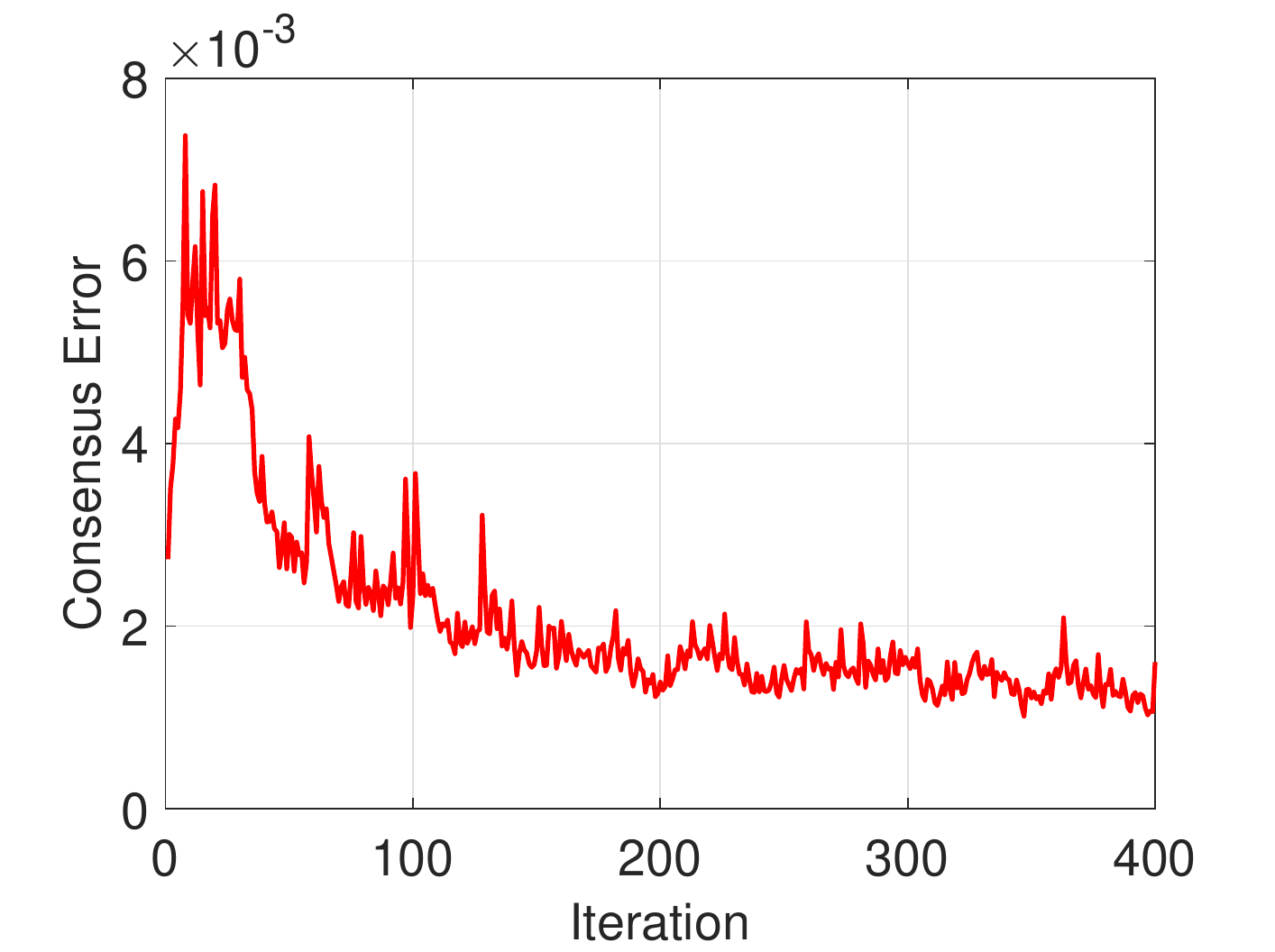}\label{fig_c}}
     	\vskip -0.05 in
     	\caption{Iteration and communication complexity in homogeneous environment: (a)(b) $N=5$, IGD ($\gamma = 0.095$), DGD ($\alpha = 0.09$), asI-ADMM ($\rho = 1, \tau = 10, \bar{\eta} = 0.8$) ; (c)(d) $N=10$, IGD ($\gamma = 0.095$), DGD ($\alpha = 0.09,$), asI-ADMM ($\rho = 1, \tau = 10, \bar{\eta} = 0.8$).}
     	\label{fig:rl_homo} 
 \end{figure*}

  \setcounter{figure}{5}
 \begin{figure*}[t] 
     	\centering
     	\vskip -0.0    in
    	\subfloat[ ]{\hspace*{-2mm}\includegraphics[width=45mm]{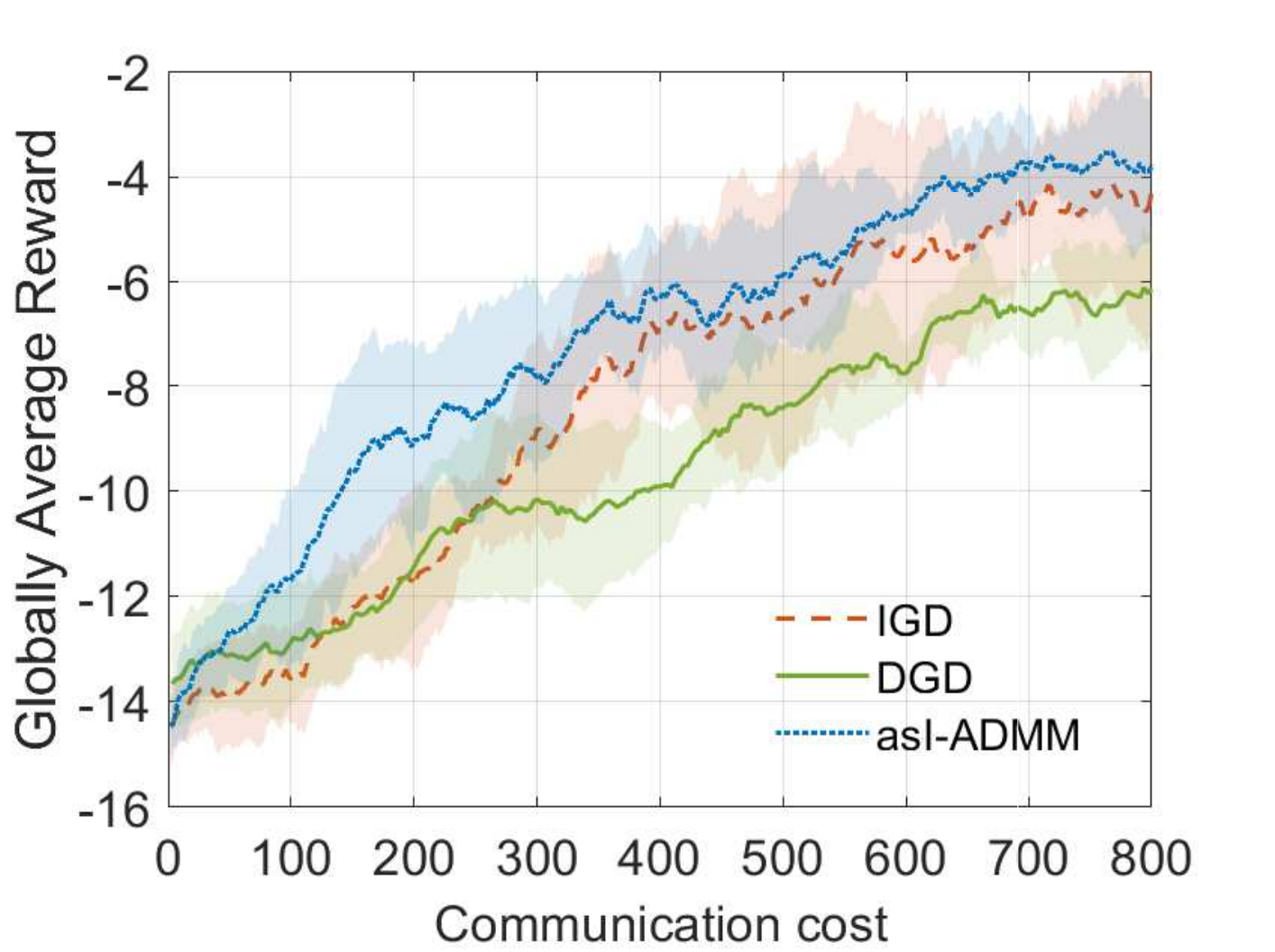}\label{fig2_a}}
     	\subfloat[ ]{\hspace*{-1mm}\includegraphics[width=45mm]{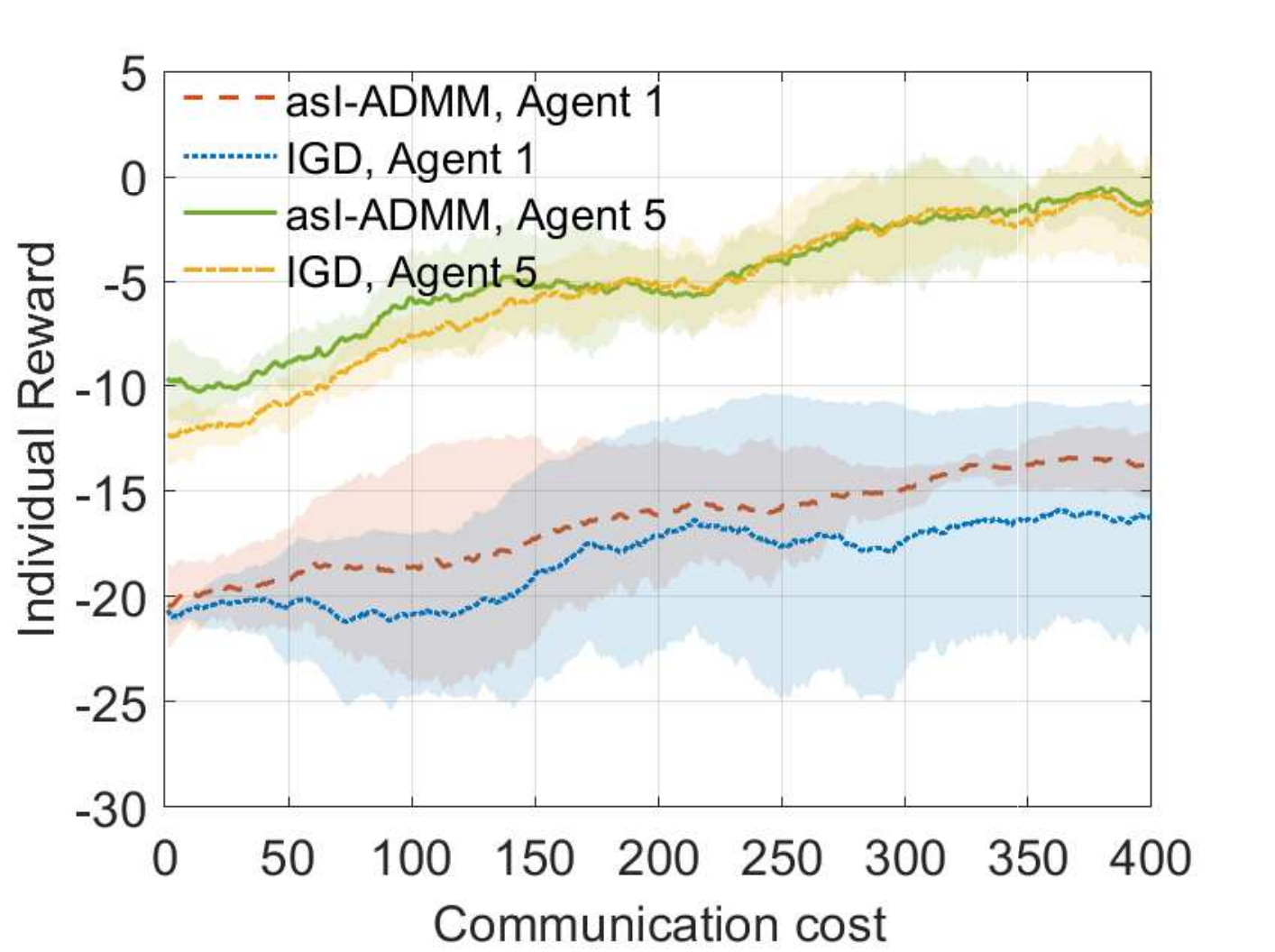}\label{fig2_b}}
     	\subfloat[ ]{\hspace*{-1mm}\includegraphics[width =45mm]{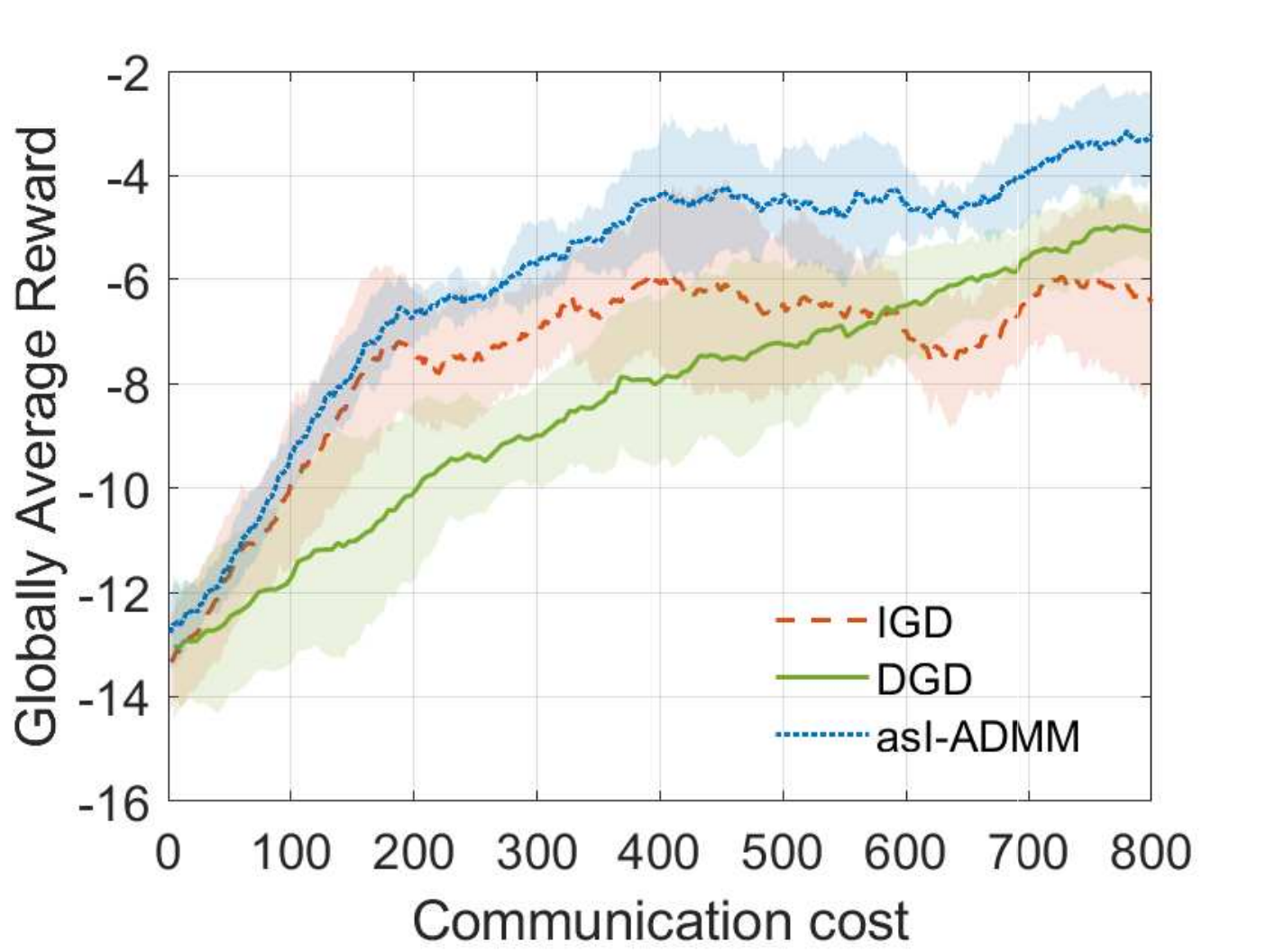}\label{fig2_c}}
     	\subfloat[ ]{\hspace*{-1mm}\includegraphics[width=45mm]{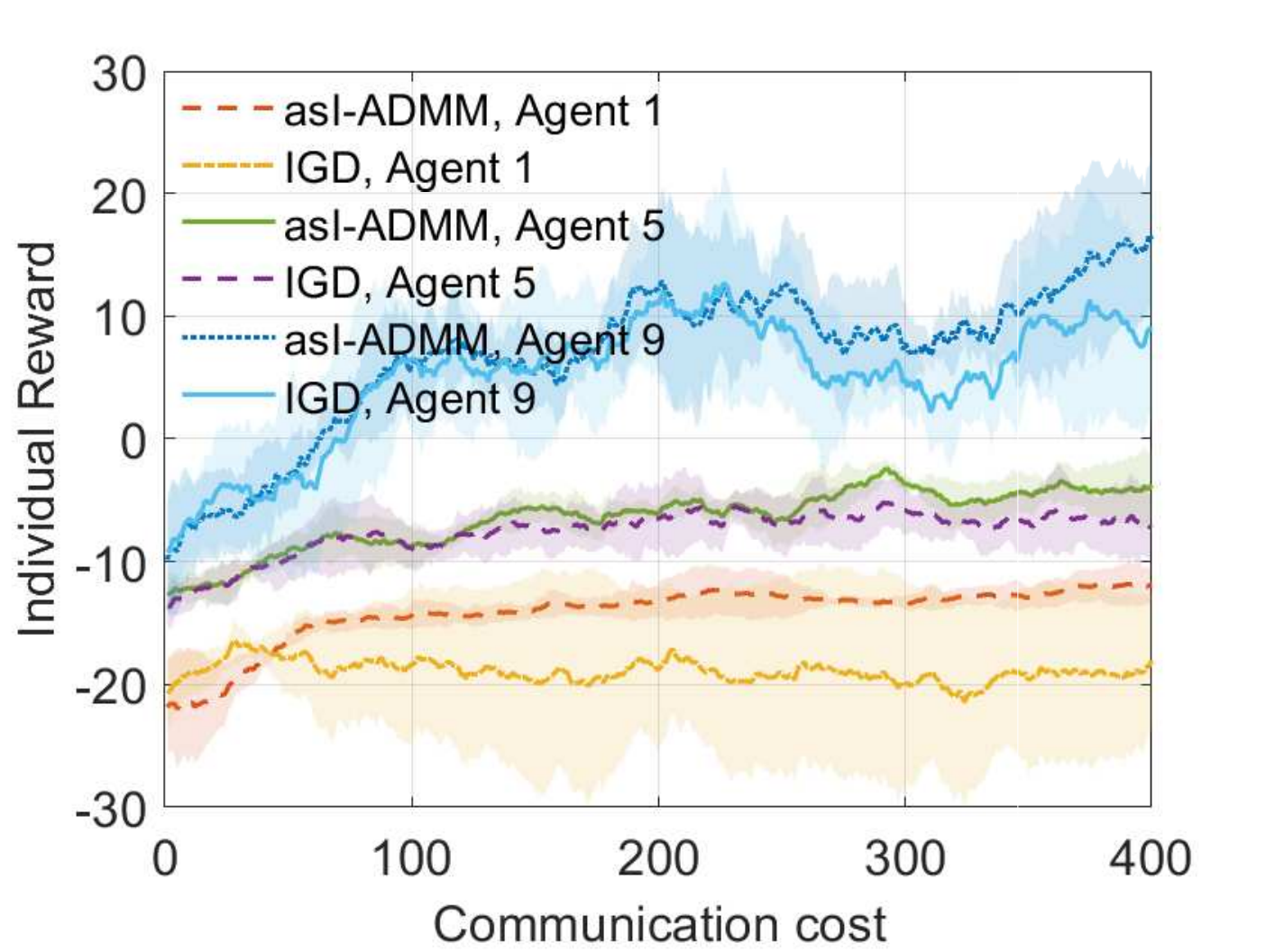}\label{fig2_d}}
     	\vskip -0.05  in
     	\caption{Iteration and communication complexity in heterogeneous environment (scaled reward and different initial state distribution): (a)(b) $N=5$, IGD ($\gamma = 0.095$), DGD ($\gamma = 0.095$), asI-ADMM ($\rho = 1, \tau = 10, \bar{\eta} = 0.8$) ; (c)(d) $N=10$, IGD ($\gamma = 0.01$), DGD ($\gamma = 0.01$), asI-ADMM ($\rho = 1, \tau = 10, \bar{\eta} = 0.8$).}
     	\label{fig:rl_hex} 
 \end{figure*}

\setcounter{figure}{3} 
\begin{figure} [t]
\centering
\vskip  0.  in
\includegraphics[width=66mm]{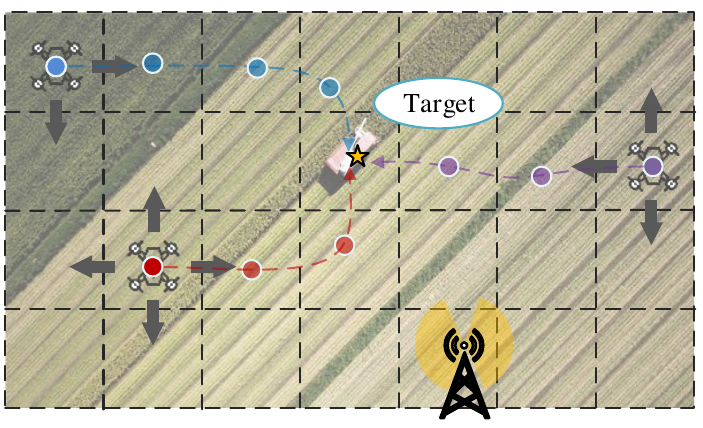}
\caption{\color{black}{Illustration of target localization in decentralized RL.}}
\label{fig:target_loc}
\vskip -0.1 in
\end{figure} 


We run 400 iterations and report globally averaged reward with 5 Monte Carlo runs. We first present the corresponding performance for the homogeneous (non-scaled reward) case in Fig. \ref{fig:rl_homo}. Clearly, our asI-ADMM converges within the same iterations as IGD \cite{bertsekas2011incremental} and DGD \cite{DGD} methods, and achieves at a better local optimum. This is because there exists a lower bound for accuracy when a fixed step size is adopted in gradient-based methods. 
We see that incremental-based algorithms (asI-ADMM and IGD) can achieve significantly better performance than DGD and yet require less communication resources. This also matches with the results as in Fig. \ref{fig:least_square}. Furthermore, as shown in Fig. \ref{fig2_c} and Fig. \ref{fig2_d}, the proposed asI-ADMM method can achieve low consensus errors at the order of $10^{-3}$ in about 200 iterations with 5-agent networks and in about 400 iterations with a 10-agent network. This is because the convergence of consensus for all learners takes longer time as the size of network increases. Results for heterogeneous networks are shown in Fig. \ref{fig:rl_hex}. The proposed asI-ADMM converges faster and obtain a better local optimum than both IGD and DGD with the same communication rounds. This is even more pronounced than in homogeneous networks. We also present an individual reward of representative agents in Fig. \ref{fig2_b} and Fig. \ref{fig2_d} (agents with a higher index have higher priority). We see that the adaptive method with exchanging $\mu$ as gradient memory helps agents to achieve faster and better consensus convergence. For example, agent 1 with lower priority and less exploring ability for initial states converges faster to the consensus in asI-ADMM compared with IGD. The result also demonstrates that the proposed algorithm can well-adapt to the complex environment with heterogeneous agents.

\setcounter{figure}{6}
 \begin{figure*}[t] 
     	\centering
     	\vskip -0.05 in
    	\subfloat[ ]{\hspace*{-2mm}\includegraphics[width =45mm]{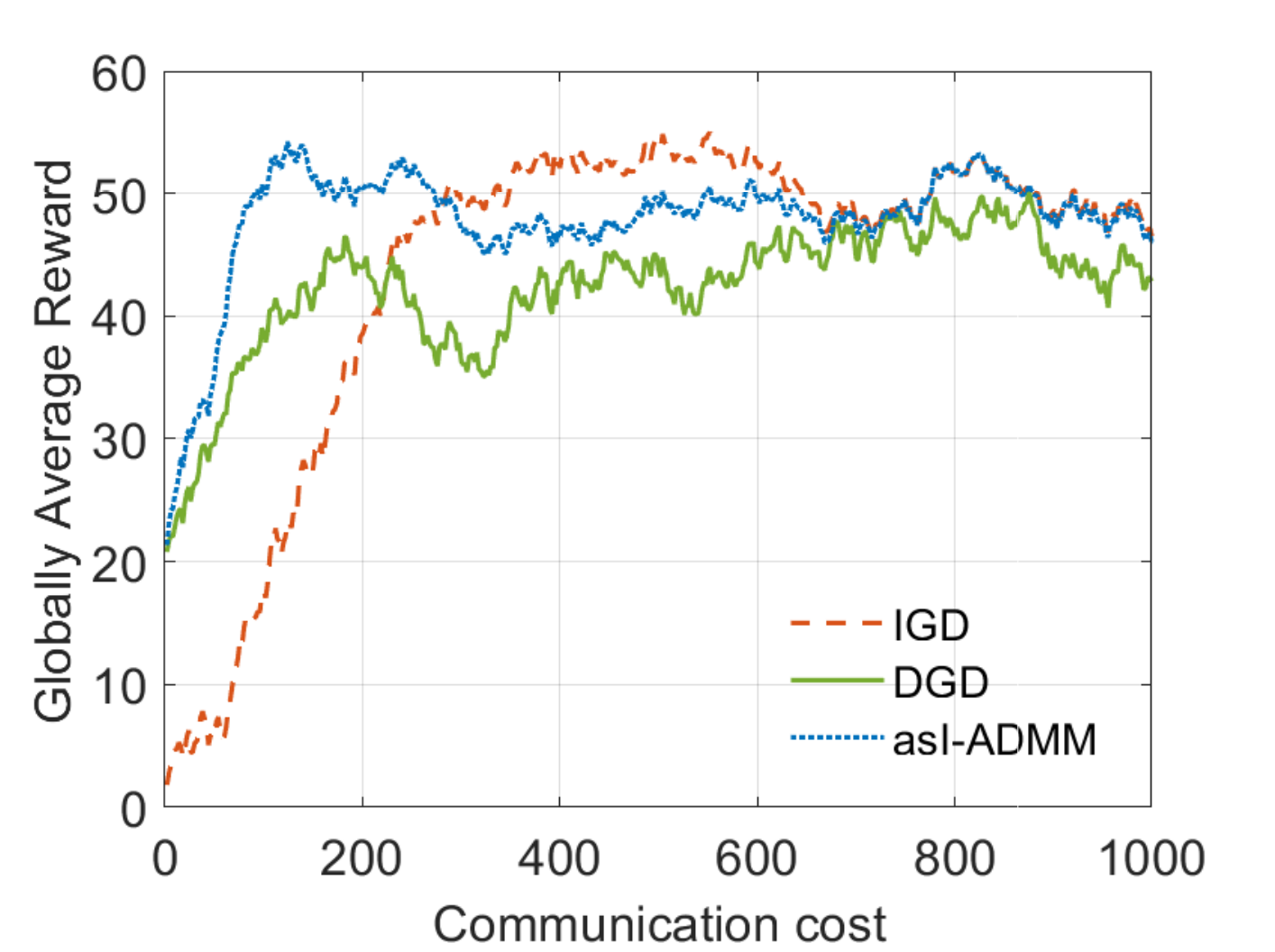}\label{figr_a}}
     	\subfloat[ ]{\hspace*{-1mm}\includegraphics[width=45mm]{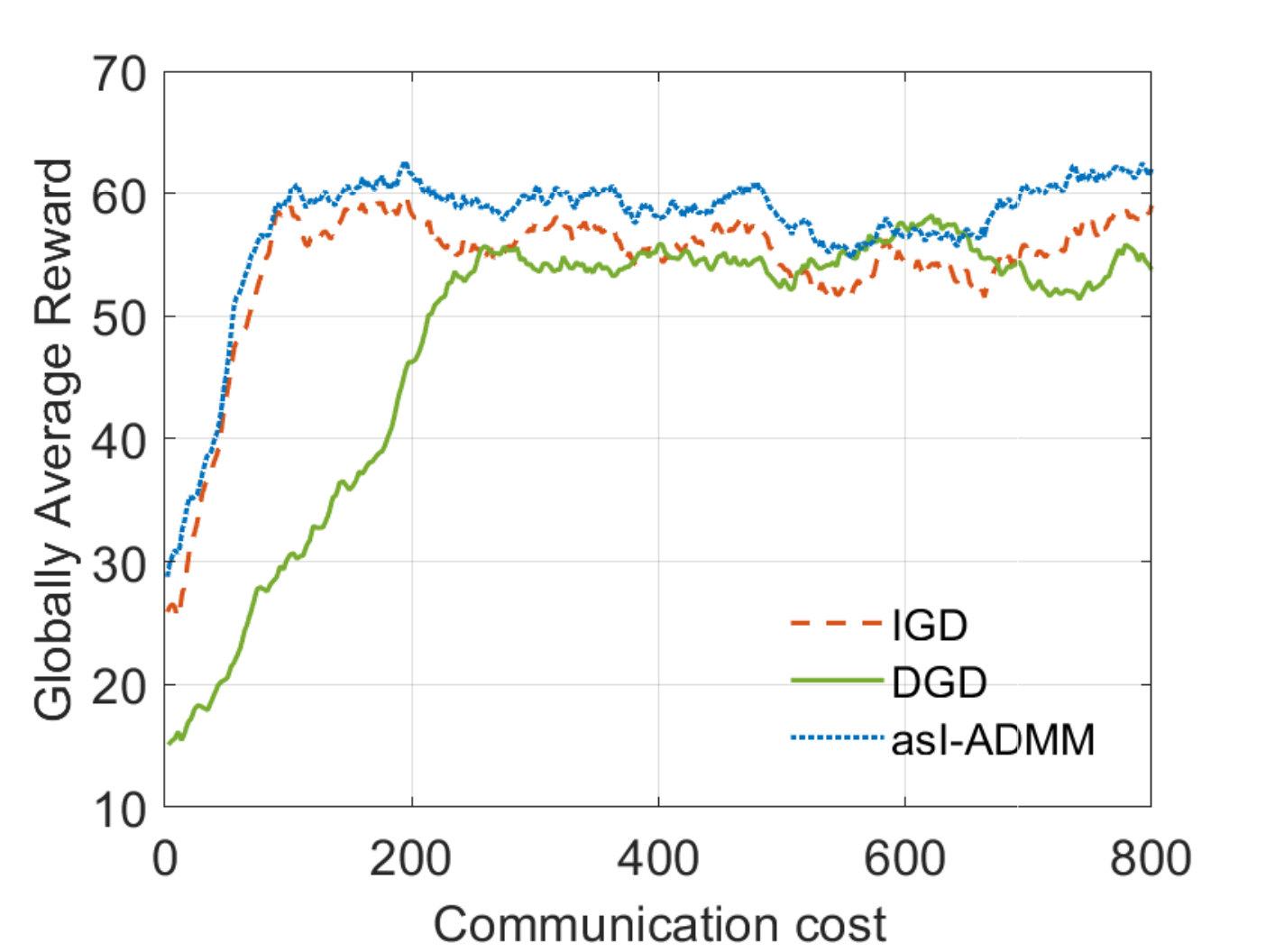}\label{figr_b}} 
     	\subfloat[ ]{\hspace*{-1mm}\includegraphics[width=45mm]{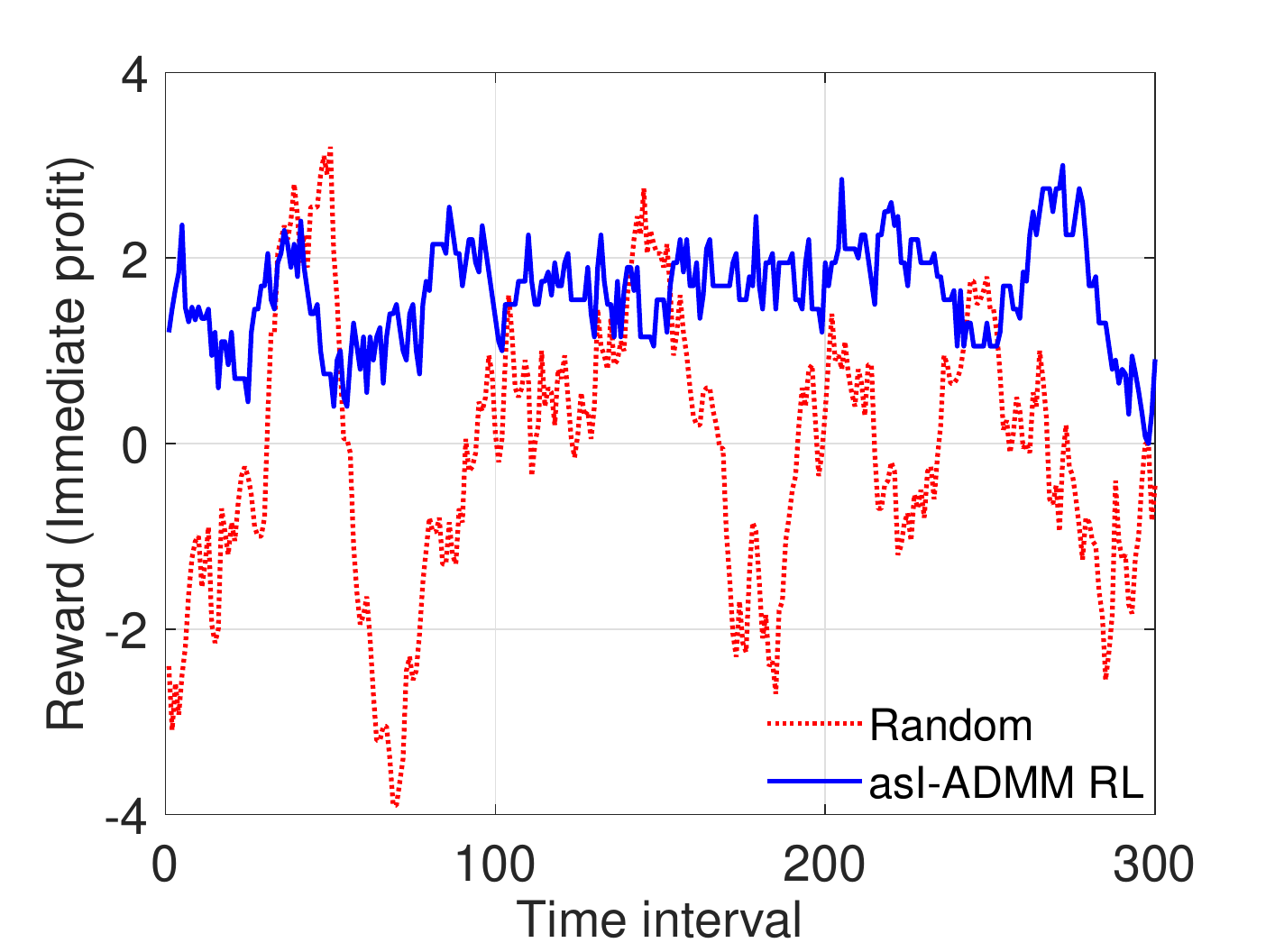}\label{figr_c}}
     	\subfloat[ ]{\hspace*{-1mm}\includegraphics[width=45mm]{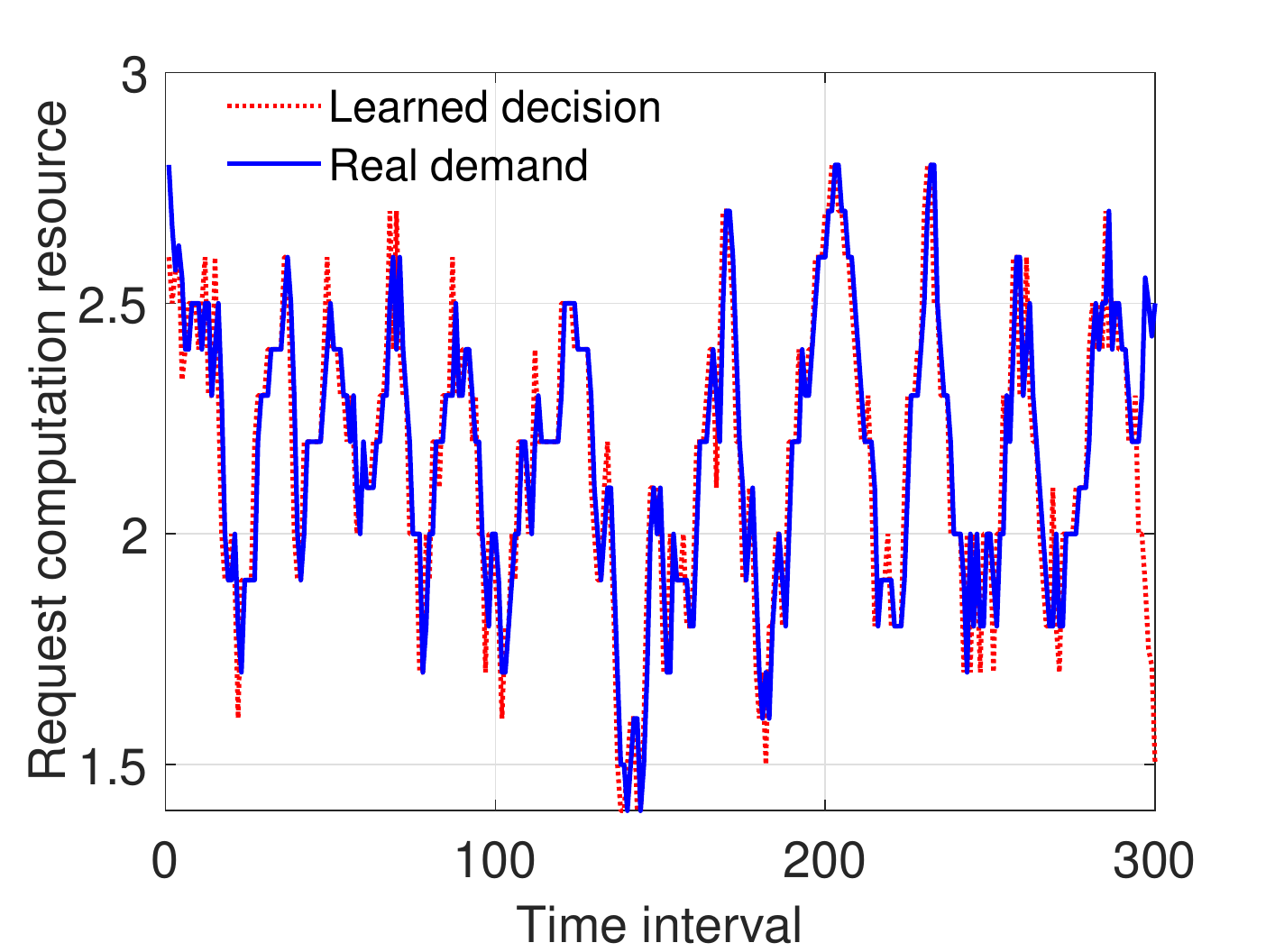}\label{figr_d}}
     	\vskip -0.05   in
     	\caption{Communication complexity and 300 seconds simulation test in computation resource management task, IGD ($\gamma = 0.0095$), DGD ($\gamma = 0.0095$), asI-ADMM ($\rho = 1, \tau = 20, \bar{\eta} = 0.8$) (a) $N=2$ (b) $N=5$.}
     	\label{fig:paral_rl} 
 \end{figure*} 
 
\subsubsection{Computation Resource Management}
 We modify the computation resource management problem in mobile edge-cloud \cite{zhang2018efficient} to a decentralized RL problem without a cloud server. 
\begin{figure} [ht]
\centering
\vskip -0   in
\includegraphics[width=76mm]{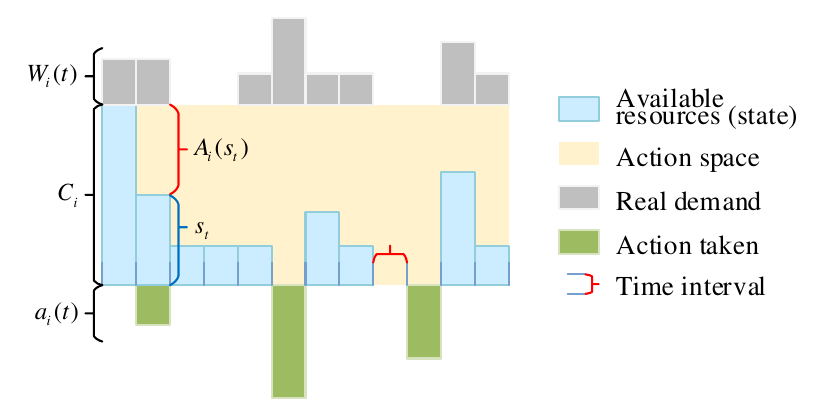}
\caption{\color{black}{Illustration of computation resources management model at the $i$-th agent .}}
\label{fig:resource_model}
\vskip -0.1 in
\end{figure} 
 Let $N$ agents work in parallel to find a common optimal policy for managing computation resource and maximize the total profit. The total computation resource at agent $i$ is denoted as $C_i$. $\mc S_i = \{0,1,...,C_i\}$ is the state space to denote the available computation resources at time interval $t$. {\color{black}We assume that the arrival rate of computation task $W_{i}(t)$ follows a Poisson distribution with expected value $d_{i}$. Thus, the probability for $w$ tasks arrive at agent $i$ is given by 
 \begin{equation}
     P\{w\} = \frac{\exp({d_i})(d_i)^w}{w!}, ~ w = 0,1,2,....
 \end{equation} 
 For each computation task, its computation workload follows an exponential distribution with an expected value of $r_i$. State transitions occur when a new task arrives, or when an old task finished. } The action space is denoted as the number of request computation resource $\mc A_i =\{0,1,...,C_i\}$ since there is no need to consider request larger than the resource capacity. We adopt the profit model in \cite{zhang2018efficient}. However, we consider only one time slot with 30 time-interval and assume that all tasks can be finished on time if there is available computation resources. The immediate reward is defined as 
  \begin{equation}
       g_i(t) =
\begin{cases}
\begin{aligned}
    -h_0 &\mathbbm I_{\{\action_t>0\}} - h_1 \state_t \\ & \quad -  h_2((\state_t +\action_t) \land C - \state_t)^+ \\ &\quad + p((\state_t + \action_t )\land C - \state_{t+1})^+  , 
\end{aligned}
 & \text{if event triggers}; \\
- h_1 \state_t , & \text{otherwise},
\end{cases}
 \end{equation}
 where $a\land b = \min(a,b)$ and $(a)^+ = \max(a,0)$, and $h_0$ is fixed costs for initiating a request and $h_2$ is the price for each resource, and $h_1$ is the cost for holding the resource, and $p$ is the price for finishing a task. The computation resource model at the $i$-th agent is shown in Fig.\ref{fig:resource_model}. We assume $d_i = d$ and $C=C_i$ for all agents.  In this way, the setup is consistent with decentralized parallel RL.

We present results as a simple example for parallel RL in Fig. \ref{fig:paral_rl}. The parameters are set by $h_0 = 4, h_1 = h_2 = 2, h_3 = 3, d=3, C=6,$ and $p=5$. From Fig. \ref{figr_a} and Fig. \ref{figr_b}, we can see asI-ADMM converges faster with communication rounds reduction with $N=2$ and $N=5$. To demonstrate the performance of the learned policy, we extract 300 time intervals from the simulation results and show the immediate profit and learned policy in Fig. \ref{figr_c} and Fig. \ref{figr_d}, respectively. It can be seen that the policy learned by asI-ADMM achieves an average of 2 profits compared with a random policy, which only achieves average 0. As seen in Fig. \ref{figr_d}, the gap between the predicted decision and the real demand is very small.  


\section{Conclusions}\label{sec:conclusion}
We study the consensus optimization problem in decentralized edge computing in IIoT using stochastic ADMM. We propose an adaptive stochastic I-ADMM (asI-ADMM) algorithm and extend the results to  decentralized RL settings. We apply a novel rule to adaptively choose the weights in first-moment estimation, which guarantees an upper bound of the estimation variance during iteration. This rule can be generally applied to other first-order approximation methods. Moreover, we provide convergence analysis and show that the proposed algorithms can achieve a $O(\frac{1}{k})+O(\frac{1}{M})$ rate. For communication-constrained environments of decentralized edge computing, our proposed algorithms work in complex heterogeneous agents using as few communication rounds as possible. We provide numerical experiments in supervised learning tasks and decentralized RL tasks and show the proposed asI-ADMM algorithms outperform the state of the art. 
\appendices
\section{Proof of Lemma \ref{lem:uppB} }
For notation simplicity, we denote $G_i(\vectheta_i^k;\zeta^k_i)$ as $G_i(\vectheta_i^k)$. Since the adaptive rule (\ref{eq:eta}) indicates $\eta^k\leq \bar{\eta}$, from updating of $\mu^{k+1}$ in (\ref{eq:mu_update}), we can have following:
 \begin{align}
     \norm{\mu^{k+1}}^2 &= \norm{\eta^k\mu^k +(1-\eta^k)G_{i_k}(\vectheta_{i_k}^k)}^2 \notag\\
     & = \norm{\prod_{t=0}^{k} \eta^t \mu^0 + \sum_{t=0}^{k} (1-\eta^t)G_{i_t}(\vectheta_{i_t}^t) \prod_{j=t}^k\eta^j }^2 \notag\\
     & \leqslant \norm{\sum_{t=0}^k (\bar{\eta})^t G_{i_t}(\vectheta_{i_t}^{t})}^2 \leqslant \frac{(1-\bar{\eta})\delta^2}{1-({\bar{\eta}})^k}.\label{eq:mubound}
 \end{align}
From the optimality condition of (\ref{eq:opti_cond}), we can derive
\begin{equation}\label{eq:state2_2}
    \mu^{k+1}-\rho \left(z^k-\vectheta_{i_k}^{k+1}+\frac{\lambda_{i_k}^k}{\rho} \right) + \tau(\vectheta_{i_k}^{k+1}-\vectheta_{i_k}^k) = 0.
\end{equation}
Then we derive the lower bound of $\mathcal L_{\rho}(\vectheta^{k+1},\veclambda^{k+1},z^{k+1})$ as follows:
\begin{align}
        &\mathbbm{E} \big[\mathcal{L}_{\rho} (\bigtheta^{k+1},\veclambda^{k+1},z^{k+1} ) \big] \notag\\
        &= \sum_{j=0}^{N-1}  \big[ f_{i_{k-j}}(\vectheta_{i_{k-j}}^{k-j+1}) +  \langle \veclambda_{i_{k-j}}^{k-j+1},z^{k+1}-\vectheta_{i_{k-j}}^{k-j+1} \rangle \notag\\
            & \qquad + \frac{\rho}{2} \norm{z^{k+1}-\vectheta_{i_{k-j}}^{k-j+1}}^2  \big], \notag\\
        &\overset{(a)}{\geqslant}  \sum_{j=0}^{N-1}  \big[ f_{i_{k-j}}(\vectheta^*) +  \langle \veclambda_{i_{k-j}}^{k-j}+ \rho\gamma(z^k - \vectheta_{i_{k-j}}^{k-j+1}) ,z^{k+1}-\vectheta_{i_{k-j}}^{k-j+1} \rangle \notag\\
        & \qquad + \frac{\rho}{2} \norm{z^{k+1}-\vectheta_{i_{k-j}}^{k-j+1}}^2  \big], \notag\\
        &\overset{(b)}{\geqslant}F^*+\sum_{j=0}^{N-1}  \big[ \frac{\rho}{2} \norm{z^{k+1}-\vectheta_{i_{k-j}}^{k-j+1}}^2 +\underline{\mathbbm{E}\big[ \langle\mu^{k+1}, z^{k+1}-\vectheta_{i_{k-j}}^{k-j}\rangle  \big]}\notag\\
        &\quad  \underline{+   \langle   
        \tau(\vectheta_{i_{k-j}}^{k-j+1}-\vectheta_{i_{k-j}}^{k-j}) + \rho(\gamma-1)(z^k- \vectheta_{i_{k-j}}^{k-j+1})
        ,z^{k+1}-\vectheta_{i_{k-j}}^{k-j+1}  \rangle}_{\mathcal{A}}  
        \label{eq:lemma1_1}
        \end{align} 
    where $(a)$ and $(b)$ hold because of equation (\ref{eq4b}) and (\ref{eq:state2_2}), respectively. Using inequality $\langle a,b\rangle \leqslant \frac{\epsilon}{2}\norm{a}^2 + \frac{1}{2\epsilon}\norm{b}^2$ for any $\epsilon\in\mathbbm{R}^+$ on term $\mathcal A$, it derives
    \begin{align}
         \mathcal A  
         & \geqslant -\frac{2}{\rho}\norm{\mu^{k+1}}^2 - \frac{2\tau^2}{\rho} \norm{ \vectheta_{i_{k-j}}^{k-j+1}-\vectheta_{i_{k-j}}^{k-j}}^2  \notag\\
         & \qquad -\frac{3\rho}{8} \norm{z^{k+1}-\vectheta_{i_{k-j}}^{k-j+1}}^2- 2\rho(\gamma-1)^2 \norm {z^k-\vectheta_{i_{k-j}}^{k-j+1}}^2\notag \\
    & \geqslant  -\frac{2(1-\bar{\eta})\delta^2}{\rho(1-({\bar{\eta}})^k)} - \frac{2\tau^2}{\rho}\norm{ \vectheta_{i_{k-j}}^{k-j+1}-\vectheta_{i_{k-j}}^{k-j}}^2  \notag\\
        &\qquad - \frac{3\rho}{8} \norm{z^{k+1}-\vectheta_{i_{k-j}}^{k-j+1}}^2 - 2\rho(\gamma-1)^2 \norm {z^k-\vectheta_{i_{k-j}}^{k-j+1}}^2.
    \label{eq:lemma1_2}
    \end{align}
Therefore for $k=0,1,2,...$, we can have
 	\begin{equation}
 	    \mathcal{V}^{k+1} \geqslant F^* -\frac{2N(1-\bar{\eta})\delta^2}{\rho(1-{(\bar{\eta}})^k)} -  \left[\frac{2\tau^2}{\rho}-\frac{3\rho}{8}-\rho(\gamma-1)^2 \right]ND_{\mathcal X},
 	\end{equation}
 	which concludes the proof for Lemma 3.

\section{Proof of Lemma \ref{lem:converg} }
Using (\ref{eq:primal_first_order}) and through some algebra operation, we have 
\begin{align}\label{eq:state1}
   &\mathbbm{E} \left[\mathcal{L}_{\rho}(\bm{\Theta}^{k+1},\bm{\lambda}^{k+1},z^k) - \mathcal{L}_{\rho}(\bm{\Theta}^{k+1},\bm{\lambda}^{k+1},z^{k+1}) \right]\notag\\
   & =\sum_{i\in\mathcal{N}}  \left[ \langle  \lambda_i^{k+1},z^k-z^{k+1}  \rangle +\frac{\rho}{2}  (\norm{z^k-\bm\theta_i^{k+1}}^2 -\norm{z^{k+1}-\bm\theta_i^{k+1}}^2 )\right] \notag \\
   & =  \frac{N\rho}{2}\norm{z^k-z^{k+1}}^2 + \rho \sum_{ i\in \mathcal{N}} \left \langle z^{k+1} - \bm\theta_i^{k+1} +
   \frac{\lambda_i^{k+1}}{\rho}, z^k-z^{k+1} \right\rangle \notag\\
   &  \overset{(a)}{=} \frac{N\rho}{2} \norm{ z^k - z^{k+1} }^2,
\end{align}  
 where $(a)$ is due to the update of $z^{k+1}$ in (\ref{eq4c}) which makes the second term equal to zero.
Based on (\ref{eq:L-smooth1}) and (\ref{eq:state2_2}), we can derive following
    \begin{align}
    &\mathbbm{E} \left[\mathcal{L}_{\rho} (\bigtheta^{k},\bm{\lambda}^{k},z^k ) - \mathcal{L}_{\rho} (\bigtheta^{k+1},\veclambda^{k+1},z^{k} )\notag  \right] \notag\\
    &= f_{i_k}(\vectheta^k) - f_{i_k}(\vectheta^{k+1}) +  \langle \lambda_{i_k}^k,z^k-\vectheta_{i_k}^k \rangle -  \langle \lambda_{i_k}^{k+1},z^k-\vectheta_{i_k}^{k+1} \rangle \notag\\
        &\quad + \frac{\rho}{2} (\norm{z^k - \vectheta_{i_k}^k}^2 - \norm{z^{k}-\vectheta_{i_k}^{k+1}}^2 ) \notag\\
    & = f_{i_k}(\vectheta_{i_k}^k) -f_{i_k}(\vectheta_{i_k}^{k+1}) + \frac{\rho }{2}\norm{\vectheta_{i_k}^k-\vectheta_{i_k}^{k+1}}^2 +  \langle \lambda_{i_k}^k,z^k-\vectheta_{i_k}^k \rangle \notag \\
        &\quad -  \langle \lambda_{i_k}^{k+1},z^k-\vectheta_{i_k}^{k+1} \rangle +\rho  \langle z^k-\vectheta_{i_k}^{k+1},\vectheta_{i_k}^{k+1} - \vectheta_{i_k}^k \rangle \notag\\
        & \quad + \left\langle \tau(\vectheta_{i_k}^{k+1} - \vectheta_{i_k}^k),\vectheta_{i_k}^{k+1} - \vectheta_{i_k}^k \right\rangle \notag \\
        & \quad+  \mathbbm{E} \left[\left\langle \mu^{k+1} -\rho \left(z^k-\vectheta_{i_k}^{k+1}+\frac{\lambda_{i_k}^k}{\rho} \right),  \vectheta_{i_k}^{k+1}-\vectheta_{i_k}^k \right \rangle \right]\notag\\
    & = f_{i_k}(\vectheta_{i_k}^k) -f_{i_k}(\vectheta_{i_k}^{k+1}) + \frac{\rho +2\tau}{2}\norm{\vectheta_{i_k}^k-\vectheta_{i_k}^{k+1}}^2  \notag\\ 
        & \quad- \langle \lambda_{i_k}^{k} , \vectheta_{i_k}^{k+1}-\vectheta_{i_k}^{k} \rangle +\langle \mu^{k+1},\vectheta_{i_k}^{k+1}-\vectheta_{i_k}^k  \rangle \notag\\ &\quad  + \langle \lambda_{i_k}^k,z^k-\vectheta_{i_k}^k \rangle -  \langle \lambda_{i_k}^{k+1},z^k-\vectheta_{i_k}^{k+1} \rangle \notag\\ 
     &\overset{(a)}{\geqslant}  \mathbbm{E} \left[\langle \mu^{k+1}-\nabla f_{i_k}(\vectheta_{i_k}^k),\vectheta_{i_k}^{k+1}-\vectheta_{i_k}^k  \rangle \right]  \notag\\
        &\quad + \frac{\rho-L +2\tau+1}{2}\norm{\vectheta_{i_k}^{k+1}-\vectheta_{i_k}^k}^2- \frac{1}{\rho\gamma} \norm{\lambda_{i_k}^{k+1}-\lambda_{i_k}^{k}}^2 \notag\\
     & \overset{(b)}{\geqslant} \frac{\rho -L +2\tau+1}{2}\norm{\vectheta_{i_k}^{k+1}-\vectheta_{i_k}^k}^2 - \frac{ \iota^2+\sigma^2}{M} \notag \\
        & \quad - \frac{2\rho}{ \gamma}  ( \norm{\vectheta_{i_k}^{k+1}-\vectheta_{i_k}^k}^2 
        +  N^2 \norm{ z^k - z^{k+1} }^2) \notag\\
     & = \left(\frac{\rho -L +2\tau+1}{2}- \frac{2\rho}{\gamma}\right)\norm{\vectheta_{i_k}^{k+1}-\vectheta_{i_k}^k}^2 - \frac{2\rho N^2}{\gamma}\norm{ z^k - z^{k+1} }^2 \notag\\
        & \quad - \frac{ \iota^2+\sigma^2}{M}, \label{eq:state2}
    \end{align}
where $(a)$ is obtained by applying (\ref{eq:L-smooth1}), and $(b)$ is obtained by applying $\langle x,y\rangle \leqslant \frac{1}{2 }\norm{x}^2 + \frac{1}{2}\norm{y}^2$ and Lemma \ref{lemma:muBound} to the term $\langle \mu^{k+1}-\nabla f_{i_k}(\vectheta_{i_k}^k),\vectheta_{i_k}^{k+1}-\vectheta_{i_k}^k  \rangle $.
Combining (\ref{eq:state1}) and (\ref{eq:state2}), we can have the following 
\begin{align} \label{eq:lagrian_eq1}
    & \mathbbm{E} \left[\mathcal{L}_{\rho}(\bm{\Theta}^{k+1},\bm{\lambda}^{k+1},z^{k+1}) \right] \leqslant  \mathbbm{E} \left[\mathcal{L}_{\rho}(\bm{\Theta}^{k},\bm{\lambda}^{k},z^{k}) \right]  +\frac{\iota^2+\sigma^2}{M} \notag\\
        & \qquad -\left(\frac{\rho-L+2\tau+1}{2} -\frac{2\rho}{\gamma} \right)\norm{\vectheta_{i_k}^{k+1}-\vectheta_{i_k}^k}^2 \notag\\
        & \qquad -\left(\frac{N\rho}{2} - \frac{2\rho N^2}{\gamma}\right) \norm{z^{k+1}-z^{k}}^2.
\end{align}


 	Letting $\tau \geqslant \frac{L +\rho-1}{2}$, $\rho \geqslant 1$ and $\gamma \geq 4N$ guarantees $\chi   = \frac{\rho-2L+2\tau+1}{2} -\frac{2\rho}{\gamma} \geqslant0$ and $  \varphi = \frac{N\rho}{2} - \frac{2\rho N^2}{\gamma}\geqslant0$.
 	Then by the Lyapunov function, we have
 	\begin{equation}
 	    \mathcal{V}^{k+1} \leqslant  \mathcal{V}^k +\frac{\iota^2+\sigma^2}{M} - \chi\norm{\vectheta_{i_k}^{k+1}-\vectheta_{i_k}^k}^2- \varphi\norm{z^{k+1} - z^k}^2 .
 	\end{equation}
  Since $\mathcal{V}^*$ is the low bound of function $\mathcal{V}^k$, by plugging (\ref{eq:lagrian_eq1}) into $\mathcal{V}^k$ and telescoping over $k$ from $0$ to $K$, we have
 	\begin{align}
 	    &\mathcal{V}^K-\mathcal{V}^0 = (\mathcal{V}^1-\mathcal{V}^0) + (\mathcal{V}^2-\mathcal{V}^1) + \cdots + (\mathcal{V}^K-\mathcal{V}^{K-1}) \notag \\
 	    & \leqslant -\chi\norm{\vectheta_{i_1}^1-\vectheta_{i_1}^0}^2 - \varphi\norm{z^1-z^0}^2 +\frac{ \iota^2+\sigma^2 }{M}\notag\\    
 	    & \quad -\chi\norm{\vectheta_{i_2}^2-\vectheta_{i_2}^1}^2 - \varphi\norm{z^2-z^1}^2 +\frac{ \iota^2+\sigma^2 }{M}\notag \\
        & \quad - \cdots -\chi\norm{\vectheta_{i_K}^{K}-\vectheta_{i_K}^{K-1}}^2 - \varphi\norm{z^{K}-z^{K-1}}^2 +\frac{\iota^2+\sigma^2}{M} \notag\\
     &\leqslant- \sum_{t=0}^K \left(\chi\norm{\vectheta_{i_t}^t-\vectheta_{i_t}^{t-1}}^2 + \varphi\norm{z^t-z^{t-1}}^2 \right) + \frac{K(\iota^2+\sigma^2)}{M}.
 	\end{align}
 	Thus, from Lemma \ref{lem:uppB}, the above inequality can lead to
 	\begin{equation}
 	    \frac{1}{K}\sum_{t=0}^K \left(\norm{\vectheta_{i_t}^t-\vectheta_{i_t}^{t-1}}^2 + \norm{z^t-z^{t-1}}^2 \right)  \leqslant \frac{\mathcal{V}^0 - \mathcal{V}^* }{K\kappa}+ \frac{\iota^2+\sigma^2}{M\kappa},
 	\end{equation}
 	where $\kappa = \min(\chi,\varphi)$.
 	This concludes the proof for Lemma \ref{lem:converg}.

    
\section{Proof of Theorem \ref{theorem1} }
 	We firstly define a variable $\Phi_k = \frac{\iota^2+\sigma^2}{M} + \norm{\vectheta_{i_k}^k-\vectheta_{i_k}^{k+1}}^2 +\norm{z^k-z^{k+1}}^2  $, then we have  
 \begin{align}
         &\mathbbm E \big[\norm{\partial_{\vectheta_i} \mc L_{\rho}(\bigtheta^{k+1},\veclambda^{k+1},z^{k+1})}^2  \big] \notag\\
         &= \mathbbm E\big[ \norm{\nabla f_{i_k}(\vectheta_{i_k}^{k+1})- \lambda_{i_k}^{k+1} - \rho(z^{k+1}-\vectheta_{i_k}^{k+1})}^2 \big] \notag\\
         & \overset{(a)}= \mathbbm E\big[ \norm{\nabla f_{i_k}(\vectheta_{i_k}^{k+1})- \lambda_{i_k}^{k+1}- \mu^{k+1} +\rho\left(z^k - \vectheta_{i_k}^{k+1}+\frac{\lambda_{i_k}^k}{\rho}\right)\notag \\
            & \qquad -\tau(\vectheta_{i_k}^{k+1} - \vectheta_{i_k}^{k}) - \rho(z^{k+1}-\vectheta_{i_k}^{k+1})}^2 \big] \notag\\
        & = \mathbbm E\big[ \norm{\nabla f_{i_k}(\vectheta_{i_k}^{k+1})- \mu^{k+1}-(\nabla f_{i_k}(\vectheta_{i_k}^{k})-\nabla f_{i_k}(\vectheta_{i_k}^{k+1}))\notag\\
            & \qquad+ (\lambda_{i_k}^{k} -\lambda_{i_k}^{k+1})      -\tau(\vectheta_{i_k}^{k+1} - \vectheta_{i_k}^{k}) 
             - \rho(z^{k}-z^{k+1})}^2 \big] \notag\\
        & \leqslant  \frac{5(\iota^2+\sigma^2)}{M} + \left(5L^2 + 5\tau^2\right)\norm{\vectheta_{i_k}^k-\vectheta_{i_k}^{k+1}}^2 + 5\norm{\lambda_{i_k}^{k+1}-\lambda_{i_k}^{k}}^2 \notag \\
            & \qquad + 5\rho^2 \norm{z^{k+1} -z^k}^2 \notag \\
        & \overset{(b)}\leqslant  \frac{5(\iota^2+\sigma^2)}{M} + \left(5L^2 + 5\tau^2+10\rho^2\right)\norm{\vectheta_{i_k}^k-\vectheta_{i_k}^{k+1}}^2  \notag \\
            & \qquad + (5\rho^2+10\rho^2N^2) \norm{z^{k+1} -z^k}^2 \notag \\
         &{\leqslant}  \left(5 + 5L^2 + 5\tau^2 +15\rho^2+ 10\rho^2N^2\right)\Phi_{k},  
     \end{align}
where equality $(a)$ is due to (\ref{eq:state2_2}) and inequality $(b)$ is due to
	\begin{equation}
 	    \norm{\lambda_{i_k}^{k+1}-\lambda_{i_k}^{k}}^2 \leqslant 2\rho^2\norm{\vectheta_{i_k}^{k+1}-\vectheta_{i_k}^{k}}^2 + 2\rho^2N^2\norm{z^{k+1}-z^{k}}^2. 
 	\end{equation}
Next, by the step 10 in Algorithm 1, for all $i\in \mc N$, we have 
  \begin{align}
          \mathbbm E&\big[\norm{\partial_{\lambda_i} \mc L_{\rho}(\bigtheta^{k+1},\veclambda^{k+1},z^{k+1})}^2\big] = \norm{z^{k+1}-\vectheta_{i_k}^{k+1}}^2  \notag\\
        &\qquad\leqslant 2\norm{z^k-z^{k+1}}^2 + 2\norm{z^{k}-\vectheta_{i_k}^{k+1}}^2 \notag\\
         &\qquad\leqslant \left(2+\frac{4N^2}{\gamma^2}\right) \norm{z^k-z^{k+1}}^2 + \frac{4}{\rho^2\gamma^2}\norm{\vectheta_{i_k}^{k+1}-\vectheta_{i_k}^{k}}^2 \notag\\
         &\qquad\leqslant \left(2+\frac{4N^2}{\gamma^2} + \frac{4}{\rho^2\gamma^2}\right)\Phi_k.
   \end{align}
From update for $z^{k+1}$ of \eqref{eq4c}, it is easy to verify that $\mathbbm E[\norm{\partial_{z} \mc L_{\rho}(\bigtheta^{k+1},\veclambda^{k+1},z^{k+1})}^2]=0$. Since $5 + 5L^2 + 5\tau^2 +15\rho^2+ 10\rho^2N^2>2+\frac{4N^2}{\gamma^2} + \frac{4}{\rho^2\gamma^2}$, we have
\begin{equation}
\begin{aligned}
        &\frac{1}{K}\sum_{k=0}^{K}\mathbbm E \big[ \norm{\nabla \mc  L_{\rho}(\bigtheta_{k}, \veclambda_{k},z_{k})}^2\big] \\             
        &\qquad\qquad \leqslant \frac{\varepsilon }{K} \sum_{k=0}^K \Phi_k  \leqslant  \frac{\varepsilon (\mc V^0 - \mc V^*)}{K\kappa}+ \frac{\varepsilon (\iota^2+\sigma^2)}{M\kappa},
\end{aligned}
\end{equation}
where $\varepsilon = 5 + 5L^2 + 5\tau^2 +15\rho^2+ 10\rho^2N^2$. Then it can be concluded that
\begin{equation}
    \frac{1}{K}\sum_{k=0}^{K}\mathbbm E \left[\nabla \mc L_{\rho}(\bigtheta_{k},\veclambda_{k},z_{k})^2 \right] \leqslant O\left(\frac{1}{K}\right)+O\left(\frac{1}{M}\right)
\end{equation}
which completes the proof for Theorem 1.

\bibliographystyle{IEEEtran}
\bibliography{IEEEabrv,Reflibb} 
\end{document}